\documentclass[accepted]{uai2025}
\usepackage{yc_package}

\usepackage[american]{babel}
\usepackage{natbib} %

\usepackage{mathtools} %
\usepackage{booktabs} %
\usepackage{tikz} %
\usepackage{authblk}

\renewcommand{\paragraph}[1]{\textbf{#1}~~}

\title{Moment Alignment: Unifying Gradient and Hessian Matching for Domain Generalization}

\author[1]{Yuen Chen}
\author[1]{Haozhe Si}
\author[2]{Guojun Zhang}
\author[1]{Han Zhao}
\affil[1]{%
    University of Illinois at Urbana-Champaign}
\affil[2]{%
University of Waterloo
}
\affil[ ]{%
  \texttt{\{yuenc2, haozhes3, 
hanzhao\}@illinois.edu, g39zhang@uwaterloo.ca}}

\begin{document}
\maketitle

\begin{abstract}
Domain generalization (DG) seeks to develop models that generalize well to unseen target domains, addressing distribution shifts in real-world applications.\ One line of research in DG focuses on aligning domain-level gradients and Hessians to enhance generalization.\ However, existing methods are computationally inefficient and the underlying principles of these approaches are not well understood.\ In this paper, we develop a theory of moment alignment for DG.\ Grounded in \textit{transfer measure}, a principled framework for quantifying generalizability between domains, we prove that aligning derivatives across domains improves transfer measure.\ Moment alignment provides a unifying understanding of Invariant Risk Minimization, gradient matching, and Hessian matching, three previously disconnected approaches.\ We further establish the duality between feature moments and derivatives of the classifier head.\ Building upon our theory, we introduce \textbf{C}losed-Form \textbf{M}oment \textbf{A}lignment (CMA), a novel DG algorithm that aligns domain-level gradients and Hessians in closed-form.\ Our method overcomes the computational inefficiencies of existing gradient and Hessian-based techniques by eliminating the need for repeated backpropagation or sampling-based Hessian estimation.\ We validate our theory and algorithm through quantitative and qualitative experiments.
\end{abstract}
\begin{table*}[h]
\vspace{-1em}
    \centering 
    \setlength\tabcolsep{3pt}
    \caption{Comparison of our method and prior algorithms.}
    \vspace{-0.5em}
    \resizebox{0.75\textwidth}{!}{
    \begin{tabular}{lcccccc}
        \toprule
                            & ERM & IRM & Fish/IGA/AND-Mask & Fishr/CORAL & HGP/Hutchinson & \textbf{CMA} \\
        \midrule
        Gradient Matching   & No  & Yes & Yes               & No    & Yes            & \textbf{Yes}  \\
        Hessian Matching    & No  & No  & No                & Yes   & Yes            & \textbf{Yes}  \\
        Closed-Form Hessian & --  & --  & --                & No    & No             & \textbf{Yes}  \\
        \bottomrule
    \end{tabular}}
    \label{tab:methods_comparison}
    \vspace{-1em}
\end{table*}
\vspace{-0.5em}
\section{Introduction}\label{sec:introduction}
\vspace{-0.5em}
Classic machine learning methods rely on the assumption that training and test data are drawn from the same distribution, typically described as being independent and identically distributed (\textit{i.i.d.}). However, the \textit{i.i.d.} assumption is often violated in real-world scenarios due to variations in sampling populations~\citep{santurkar_breeds_2020}, temporal changes~\citep{shankar_image_2019}, and geographic differences~\citep{hansen_high-resolution_2013, christie_functional_2018}.\ Performance degradation due to distribution shifts is particularly critical in high-stake applications. For instance, an autonomous driving system~\citep{dai_dark_2018, hu_causal-based_2021} trained on data collected in the United States may encounter different traffic conditions when deployed in other regions.\ Similarly, in medical imaging~\citep{wachinger_detect_2021,albadawy_deep_2018,tellez_quantifying_2019}, models trained on data from one demographic group may face challenges when applied to a different demographic.

Domain generalization (DG) aims to tackle this issue by leveraging data from multiple source domains to learn a model that performs well on unseen but related target domains. Although various approaches have been studied to address the DG problem, including Invariant Risk Minimization (IRM)~\cite{arjovsky_invariant_2020}, gradient matching~\citep{shi_gradient_2021,koyama_when_2021,parascandolo_learning_2020}, Hessian matching~\citep{rame_fishr_2022,hemati_understanding_2023}, and domain-invariant feature representation learning~\citep{ben-david_theory_2010, li_domain_2018,tzeng_adversarial_2017, hoffman_cycada_2017, muandet_domain_2013, long_learning_2015, zhao_learning_2019}, these methods often appear disconnected and are based on different underlying principles. We discuss these related research in~\Cref{app:related_works}.\looseness=-1

We unify these seemingly disparate methods through the theory of moment alignment. Our theory builds upon \textit{transfer measure}, a principled DG framework proposed by~\citet{zhang_quantifying_2021}. We first extend the definition of transfer measure to multi-source DG, inducing a target error bound. We then prove that aligning the derivatives improves transfer measure under different assumptions: when there exists a classifier that is simultaneously optimal across all domains (referred to as the \textit{IRM assumption}), and when there is not. We show that IRM, gradient matching, and Hessian matching approaches are special cases of moment alignment. Our theory explains the success of state-of-the-art methods like HGP and Hutchinson's algorithm~\citep{hemati_understanding_2023}, which perform both gradient and Hessian matching. This combined approach provides an advantage over methods that only match gradients or Hessians. Furthermore, we establish the duality between feature moments and the derivatives of the classifiers, thereby unifying these approaches.

Drawing from the theoretical results, we proposed \textbf{C}losed-Form \textbf{M}oment \textbf{A}lignment (CMA), a novel algorithm to DG that aligns the first- and second-order derivatives across domains. The loss objective in CMA is similar to those of HGP and Hutchinson's, but CMA enjoys computational efficiency by analytically computing gradients and Hessians. Our method bypasses the computational limitations of existing gradient and Hessian matching techniques that rely on repeated backpropagation or sampling-based estimation. Additionally, we provide two Hessian computation methods—direct Frobenius norm computation for faster performance at higher memory cost, and a memory-efficient method that reduces memory requirement at the expense of increased computation time. This flexibility allows users to balance memory usage and computational time.

The empirical evaluation consists of two settings designed to validate our theoretical framework and proposed algorithm. First, we conduct linear probing experiments on Waterbirds, CelebA, and MultiNLI datasets, where the IRM assumption holds. Second, we perform full fine-tuning experiments on selected datasets from the DomainBed benchmark~\citep{gulrajani_search_2020}, where the IRM assumption may not be satisfied. In the DomainBed experiment, where the IRM assumption is not guaranteed. We compare CMA with ERM, CORAL~\citep{sun_deep_2016}, and Fishr~\citep{rame_fishr_2022}. CMA's performance aligns with our theory and matches state-of-the-art performance.

Below we summarize our main contributions:
\begin{itemize}[noitemsep,topsep=0pt]
    \item \textit{Unified Theory of Moment Alignment:} We develop a theory of moment alignment that unifies IRM, gradient matching, and Hessian matching.\ This unified framework enhances our understanding of the interplay between these methods and their combined effect on improving generalization across domains.\ We further establish the duality between feature moments and the classifier derivatives.

    \item \textit{New Algorithm:}\ We propose \textbf{C}losed-Form \textbf{M}oment \textbf{A}lignment (CMA), a novel DG algorithm that performs both gradient and Hessian matching. CMA enjoys computational efficiency by analytically computing gradients and Hessians, avoiding the need for repeated backpropagation or sampling-based estimation.\ We offer two Hessian computation methods to optimize memory usage and computational speed.
    \item \textit{Empirical Validation:} We validate CMA through both quantitative and qualitative analyses. CMA matches state-of-the-art performance while achieving superior worst-group accuracy and feature moment alignment, reducing first- and second-moment discrepancies more effectively than Fishr and ERM.

\end{itemize}

Our work offers a unified perspective that enhances theoretical understanding and practical performance in addressing distribution shifts. As summarized in \Cref{tab:methods_comparison}, our method is, to the best of our knowledge, the first to achieve exact gradient and Hessian matching.

\vspace{-0.75em}
\section{Preliminaries}\label{sec:background}
\vspace{-0.75em}
We consider the problem of DG, where predictors are trained on data drawn from a set of source domains and are evaluated on an unseen target domain. The goal is to learn a predictor that generalizes well to the target domain. Formally, the data are drawn from $K$ source domains $\mathcal{S} := \{\mu_1, \ldots, \mu_K\}$ and a target domain $\mathcal{T} := \mu_{\mathcal{T}}$. Each domain $\mu_i$ is a distribution over the input space $\mathcal{X}$ and the label space $\mathcal{Y}$. The loss of a predictor $h: \mathcal{X} \rightarrow \mathcal{Y}$ on domain $\mu$ is defined as $\mathcal{L}_{\mu}(h ) = \mathbb{E}_{(\mathbf{x}, y) \sim \mu} [\ell(h(\mathbf{x}), y)]$, where $\ell$ is the loss function on a single example. The goal of domain generalization is to learn $h \in \mathcal{H}$ to minimize the loss on the target domain $\mathcal{T}$: $\mathcal{L}_{\mu_\mathcal{T}}(h) = \mathbb{E}_{(\mathbf{x}, y) \sim \mu_\mathcal{T}} [\ell(h(\mathbf{x}), y)]$. 
ERM minimizes the average loss over the source domains, $\mathcal{L}_{\text{ERM}} := \frac{1}{K} \sum_{i=1}^K \mathcal{L}_{\mu_i}$.\ However, ERM often fails under distribution shifts, especially when the data exhibits spurious correlation.\ To address this, \citet{arjovsky_invariant_2020} proposes the IRM principle, aiming to jointly learn a features extractor and a predictor such that there exists a predictor on the extracted features that is optimal for all domains simultaneously. Subsequent studies, such as those by \citet{rosenfeld_risks_2021} and \citet{wang_provable_2022, wang_invariant-feature_2023}, have shown that IRM alone is not sufficient for DG. Recent work by \citet{zhang_quantifying_2021} on \textit{transferability} introduces a framework to measure how much success a predictor trained on one domain can transfer to another. Below we review the original definition of transfer measures between two domains and extend it to multi-source domain settings.

\vspace{-0.5em}
\subsection{Transfer Measures}\label{sec:transfer_measures}
\vspace{-0.5em}
We restate the definitions of \textit{transfer measures}~\citep{zhang_quantifying_2021} and its induced target error bound.
\begin{definition}[\textbf{transfer measures \citep{zhang_quantifying_2021}}]
    \label{def:quantifiable_transfer_measures} Given some $\Gamma \subseteq \mathcal{H}$, $\mathcal{L}_{\mathcal{S}}^*:=\inf _{h \in \Gamma} \mathcal{L}_{\mathcal{S}}(h)$ and $\mathcal{L}_{\mathcal{T}}^*:=\inf _{h \in \Gamma} \mathcal{L}_{\mathcal{T}}(h)$,
    we define the one-sided transfer measure, symmetric transfer measure, and the realizable transfer measure respectively as:
    \begin{equation}
        \begin{aligned}
            \mathrm{T}_{\Gamma}(\mathcal{S} \| \mathcal{T})            := & \sup _{h \in \Gamma} \mathcal{L}_{\mathcal{T}}(h)-\mathcal{L}_{\mathcal{T}}^*-\left(\mathcal{L}_{\mathcal{S}}(h)-\mathcal{L}_{\mathcal{S}}^*\right)                                                                                                                                  \\
            \mathrm{T}_{\Gamma}(\mathcal{S}, \mathcal{T})              := & \max \left\{\mathrm{T}_{\Gamma}(\mathcal{S} \| \mathcal{T}), \mathrm{T}_{\Gamma}(\mathcal{T} \| \mathcal{S})\right\}\\
            = & \sup _{h \in \Gamma}\left|\mathcal{L}_{\mathcal{S}}(h)-\mathcal{L}_{\mathcal{S}}^*-\left(\mathcal{L}_{\mathcal{T}}(h)-\mathcal{L}_{\mathcal{T}}^*\right)\right| \\
            \mathrm{T}_{\Gamma}^{\mathrm{r}}(\mathcal{S}, \mathcal{T}):= & \sup _{h \in \Gamma}\left|\mathcal{L}_{\mathcal{S}}(h)-\mathcal{L}_{\mathcal{T}}(h)\right|
        \end{aligned}
    \end{equation}
\vspace*{-1.5em}
\end{definition}
From the definition of one-sided transfer measure, we have the following target error bound.
\begin{restatable}[\textbf{target error bound~\citep{zhang_quantifying_2021}}]{proposition}{proptebbound}\label{prop:target_error_bound}
    For any $h \in \Gamma\subseteq \mathcal{H}$, the target error is bounded by:
    \begin{equation}
    \begin{aligned}
        \mathcal{L}_{\mathcal{T}}(h) \leq & \mathcal{L}_{\mathcal{S}}(h)+\mathcal{L}_{\mathcal{T}}^*-\mathcal{L}_{\mathcal{S}}^*+\mathrm{T}_{\Gamma}(\mathcal{S} \| \mathcal{T}) 
    \end{aligned}
    \end{equation}
\end{restatable}
The implication of \Cref{prop:target_error_bound} is that by minimizing the loss on the source domain and the one-sided transfer measure between the source and target domains, we can effectively minimize an upper bound on the target loss. 
\vspace{-0.5em}
\subsection{Approximate Hessian Alignment}\label{sec:approx_hessian_alignment}
\vspace{-0.5em}
\citet{hemati_understanding_2023} proves an upper bound on the transfer measure by the spectral norm of the Hessian matrices between source and target domains and is the first to propose simultaneously aligning gradients and Hessians. However, their analysis is limited to the single source domain adaptation setting and assumes the existence of an \textit{invariant optimal predictor}.

\begin{definition}[\textbf{invariant optimal predictor~\citep{arjovsky_invariant_2020}}] \label{IRM_def}
    A predictor $h \in \mathcal{H}$ is an invariant optimal predictor if $\mathcal{L}_{\mu_i}(h) = \min_{h \in \mathcal{H}} \mathcal{L}_{\mu_i}(h)$ for all $i \in [K]$.
\end{definition}
\begin{assumption}[\textbf{IRM assumption}]\label{assumption:IRM}
    There exists an invariant optimal predictor $h \in \mathcal{H}$ on the source domains $\mathcal{S} = \{\mu_i\}_{i=1}^K$.
\end{assumption}
The algorithms in \citet{hemati_understanding_2023} approximate the Hessian matrices. Both methods are computationally intensive: Hessian-Gradient Product (HGP) requires repeated backpropagation, whereas Hutchinson's method relies on estimation through sampling.

In this work, we extend the analysis of Hessian alignment to DG, addressing scenarios both with and without the IRM assumption. We also propose a more efficient algorithm that analytically computes the Hessian matrices with respect to (w.r.t.) the classifier head.
\vspace{-0.5em}
\subsection{Nature of Distribution Shift}
\vspace{-0.5em}
In the DG literature, there are two main types of assumptions on the underlying data-generating process and the nature of the distribution shift.

The first type relies on causal graphs (directed graphical models) to explicitly model the ground-truth data-generating distribution, over which one can also aim for the minimax out-of-distribution generalization performance using the invariant predictor principle~\citep{peters_causal_2015, arjovsky_invariant_2020, wang_invariant-feature_2023, zhang_causal_2023}. However, these explicit assumptions on the causal structure of the variables are often too restrictive and hard to verify in practice, due to the unobserved confounders. 

The second type of assumption explicitly models the nature of the distribution shift, such as covariate shift, label shift, concept shift~\citep{ben-david_theory_2010,heinze-deml_invariant_2018,zhao_learning_2019}.\ These assumptions make technical analysis possible but often oversimplify the true real-world shifts, which rarely adhere strictly to such constraints. Moreover, these assumptions are sufficient but not necessary for provable OOD generalizations.

Given the limitations of these two types of assumptions, our work aims to broaden its potential applicability by avoiding explicit assumptions on the underlying data-generating distributions and the nature of distribution shifts. Instead, our approach focuses on the loss landscapes of the train and test distributions, which are more fine-grained and fundamental. We would also like to point out that typical explicit distribution shift assumptions, such as the covariate shift assumption, which is closely related to the line of work on invariant risk minimization, can be used to simplify certain terms in our generalization upper bound.

\vspace{-0.5em}
\section{Theory of Moment Alignment}\label{sec:theory}
\vspace{-0.5em}
In this section, we first extend the transfer measures to multi-source domain generalization (\Cref{sec:theory_extend_df}) and prove a bound on the transfer measure independent of the target distribution (\Cref{sec:tm_half}). We then apply this bound and \Cref{prop:target_error_bound} to show that aligning derivatives across domains minimizes the target loss both under the IRM assumption (\Cref{sec:theory_irm}), and when it does not hold (\Cref{sec:theory_no_irm}).\ We defer the proof of propositions, theorems, and corollaries to \Cref{app:proof_prop}, \Cref{app:proof_thm}, and \Cref{app:proof_coro} respectively.
\vspace{-0.25em}
\subsection{Transfer Measures for Multi-Source Domains}\label{sec:theory_extend_df}
\vspace{-0.25em}
The original definition of transfer measures is defined only for a single source domain $\mathcal{S}$ and a target domain $\mathcal{T}$. Next, we first state the generalized definition to multiple source domains $\mathcal{S} =  \left\{\mu_i\right\}_{i=1}^K$.

\begin{restatable}[\textbf{transfer measures on multiple source domains}]{definition}{def_irm}\label{def:multi_transfer_measures}
    Given $\mathcal{S} =  \left\{\mu_i\right\}_{i=1}^K$, some $\Gamma \subseteq \mathcal{H}$, $\mathcal{L}_{\mu_i}^* : = \inf_{h \in \Gamma} \mathcal{L}_{\mu_i}(h)$ for all $i \in [K]$, $\mathcal{L}_\mathcal{T}^* : = \inf_{h \in \Gamma} \mathcal{L_T}(h)$, $\mu^* := \argmin_{\substack{\mu}} \max_{\substack{i \in [K]}}
        \mathrm{T}_{\Gamma} \left(\mu_i\| \mu\right)$, and $\mathcal{L}_\mathcal{S} \left(h\right) := \mathcal{L}_{\mu^*} \left(h\right)$. we define the one-sided transfer measure, symmetric transfer measure, and the realizable transfer measure respectively as:
    \begin{equation}
        \begin{aligned}
            \mathrm{T}_{\Gamma}(\mathcal{S} \| \mathcal{T})
             := & \sup _{h \in \Gamma} \mathcal{L}_{\mathcal{T}}(h)-\mathcal{L}_{\mathcal{T}}^*-\left(\mathcal{L}_{\mathcal{S}}(h)-\mathcal{L}_{\mathcal{S}}^*\right)  \\
             = &\sup_{h \in \Gamma} \mathcal{L}_{\mathcal{T}}(h) - \mathcal{L}_{\mathcal{T}}^* - \left(\mathcal{L}_{\mu^*}(h) - \mathcal{L}_{\mu^*}^*\right)  \\
              = &\mathrm{T}_{\Gamma}(\mu^* \| \mathcal{T}) \nonumber                                                                                                                                                                                                                                                 \\
            \mathrm{T}_{\Gamma}(\mathcal{S}, \mathcal{T})
              := & \max\left\{\mathrm{T}_{\Gamma}(\mathcal{S} \| \mathcal{T}), \mathrm{T}_{\Gamma}(\mathcal{T} \| \mathcal{S})\right\}  \\
              = &\mathrm{T}_{\Gamma}(\mu^*, \mathcal{T}) \nonumber                                                                                              \\
            \mathrm{T}_{\Gamma}^{\mathrm{r}}(\mathcal{S}, \mathcal{T})
              := & \sup _{h \in \Gamma}\left|\mathcal{L}_{\mathcal{S}}(h)-\mathcal{L}_{\mathcal{T}}(h)\right|  \\
              = & \sup _{h \in \Gamma}\left|\mathcal{L}_{\mu^*}(h)-\mathcal{L}_{\mathcal{T}}(h)\right|                                                                                      = \mathrm{T}_{\Gamma}^{\mathrm{r}}(\mu^*, \mathcal{T})
        \end{aligned}
    \end{equation}
\end{restatable}

In words, we define the transfer measure between a set of domains $\mathcal{S}$ and a domain $\mathcal{T}$ as the transfer measure between a distribution $\mu^*$ and $\mathcal{T}$, where $\mu^*$ is the center of mass. Note that it is not necessary to find $\mu^*$ explicitly, as shown next. For the remainder of this paper, we use transfer measure to refer to the one-sided transfer measure and leave the analogous results for the symmetric and realizable transfer measures to \Cref{app:other_tm}. \looseness=-1

\vspace{-0.25em}
\subsection{Bounding Transfer Measure}\label{sec:tm_half}
\vspace{-0.25em}
Although \Cref{def:multi_transfer_measures} defines the transfer measure in the multiple source domain setting, in DG, one can only access the source domains $\mathcal{S} = \left\{\mu_i\right\}_{i=1}^K$. In this section, we prove an upper bound on the transfer measure on the target domain $\mathcal{T}$ under the following mixture assumption:\looseness=-1
\begin{assumption}[\textbf{convex combination of source domains}]\label{assumption:convex_combination}
    The target domain $\mathcal{T}$ is a convex combination of the source domains $\mathcal{S} = \left\{\mu_i\right\}_{i=1}^K$, i.e., $\exists w_i \geq 0$ and $\sum_{i=1}^K w_i = 1$ such that $\mu_{\mathcal{T}} = \sum_{i=1}^K w_i \mu_i$.
\end{assumption}

Note that although \Cref{assumption:convex_combination} seems restrictive, this assumption generally holds. In particular, the assumption is satisfied in the simple case where each source distribution is a Gaussian. Moreover, as a well-known result in the literature of mixture models shows, when the number of mixture components is large enough, any smooth continuous distribution can be well-approximated by a mixture of Gaussians. Furthermore, in the literature of domain generalization, similar assumptions have been adopted as well for the purpose of analysis and design of algorithms~\citep{hu_does_2018, sagawa_distributionally_2020,krueger_out--distribution_2021}.\looseness=-1

From the definition of transfer measure and \Cref{assumption:convex_combination}, we have the following proposition.
\begin{restatable}[\textbf{upper bound on transfer measure}]{proposition}{propuppertm}\label{prop:tm_half}
    Given $\mathcal{S} = \left\{\mu_i\right\}_{i=1}^K$ and some $\Gamma \subseteq \mathcal{H}$. Define $\mathcal{L}_{\mu_i}^* : = \inf_{h \in \Gamma} \mathcal{L}_{\mu_i}(h)$ for all $i \in [K]$, $\mathcal{L_T}^* : = \inf_{h \in \Gamma} \mathcal{L_T}(h)$, $\mu^* := \arg\min_{\substack{\mu}} \max_{\substack{i \in [K]}}
        \mathrm{T}_{\Gamma} \left(\mu_i\| \mu\right)$, and $\mathcal{L}_\mathcal{S} \left(h\right) := \mathcal{L}_{\mu^*} \left(h\right)$. Under \Cref{assumption:convex_combination}, we have:
    \begin{equation}
        \label{eq:one_sided_upperbound}
        \mathrm{T}_{\Gamma}(\mathcal{S} \| \mathcal{T})
        \leq \frac{1}{2}\max _{i \neq j} \mathrm{T}_{\Gamma}\left(\mu_j \| \mu_i\right)
    \end{equation}
\end{restatable}
A direct consequence of \Cref{prop:tm_half} is a target error bound analogous to \Cref{prop:target_error_bound}, but does not require the knowledge of the target domain $\mathcal{T}$ except that it is a mixture of the source distributions.

\begin{restatable}[\textbf{target error bound -- multiple source domains}]{proposition}{proptemulti}\label{prop:target_error_bound_half}
    Given $\Gamma \subseteq \mathcal{H}$, for any $h \in \Gamma$, the target error is bounded by:
    \begin{equation}
                \mathcal{L}_{\mathcal{T}}(h) \leq \mathcal{L}_{\mathcal{S}}(h)+\mathcal{L}_{\mathcal{T}}^*-\mathcal{L}_{\mathcal{S}}^*+\frac{1}{2}\max _{i \neq j} \mathrm{T}_{\Gamma}\left(\mu_j \| \mu_i\right)
    \end{equation}
\end{restatable}

Having established the transfer measure provides an upper bound on the target error, we now focus on bounding the transfer measure.

\vspace{-0.25em}
\subsection{Moment Alignment under IRM Assumption}\label{sec:theory_irm}
\vspace{-0.25em}
For ease of notation, we assume that the classifier $h \in \mathcal{H}$ is parameterized by $\boldsymbol{\theta} \in \mathbb{R}^d$ so that $\mathcal{L}_\mu(h) = \mathcal{L}_\mu\left(\boldsymbol{\theta}\right)$.

\begin{restatable}[\textbf{moment alignment under IRM}]{theorem}{thmirm}\label{thm:moment_alignment_irm}
    Given $K$ source domains $\mathcal{S} = \left\{\mu_i\right\}_{i=1}^K$ and a target domain $\mathcal{T} \in conv(\mu_1, \dots, \mu_k)$, assume the losses $\mathcal{L}_{\mu_i} \: \forall i \in [K]$ are $\nu$-strongly convex w.r.t. the classifier head and $M$-times differentiable. Under the IRM assumption (\Cref{assumption:IRM}), let $\boldsymbol{\theta}^*$ be the optimal invariant predictor, $\Gamma = \argmin(\mathcal{L_S}, \delta_\mathcal{S})_\mathcal{H} := \left\{\boldsymbol{\theta} \mid h_{\boldsymbol{\theta} \in \mathcal{H}}: \max_{i \in [K]} \left(\mathcal{L}_{\mu_i} \left(\boldsymbol{\theta}\right) - \mathcal{L}_{\mu_i} \left(\boldsymbol{\theta}^*\right)\right) \leq \delta_{\mathcal{S}}\right\}$, and $\delta = \frac{2\delta_\mathcal{S}}{\nu}$, we have: 
    \begin{equation}
    \begin{aligned}
        \mathrm{T}_\Gamma (\mathcal{S}\|\mathcal{T}) \leq 
        &\ \max_{i \neq j} \bigg(\sum_{n = 2}^N  \frac{1}{n!} \delta^{\frac{n}{2}} \|\nabla^n_{\boldsymbol{\theta}}\mathcal{L}_{\mu_j}\left(\boldsymbol{\theta}^*\right) \\
         - &\ \nabla^n_{\boldsymbol{\theta}}\mathcal{L}_{\mu_i}\left(\boldsymbol{\theta}^*\right)\|_F^n \bigg) + o(\delta^{\frac{N}{2}})
    \end{aligned}
\end{equation} 

for any integer $2 \leq N \leq M$. $\mathtt{\nabla^n_{\boldsymbol{\theta}}\mathcal{L}(\boldsymbol{\theta})}$ is an $n^{\text{th}}$ order tensor with dimension $d \times \dots \times d$ ($n$ times)  where $\mathtt{\nabla^n_{\boldsymbol{\theta}}\mathcal{L}(\boldsymbol{\theta})}_{(k_1, \cdots, k_n)} = \frac{\partial^n \mathcal{L}(\boldsymbol{\theta})}{\partial \boldsymbol{\theta}_{k_1} \dots \partial \boldsymbol{\theta}_{k_n}}$.
\end{restatable}

The implication of \Cref{thm:moment_alignment_irm} is that the transfer measure is upper-bounded by the sum of the differences in higher-order derivatives across domains.\ Specifically, this suggests that aligning higher-order moments of the loss function promotes domain generalization.

Consider the special case when $N = 2$, we recover an upper bound similar to the Hessian alignment theorem in \citet{hemati_understanding_2023}, which we state next as a corollary.\looseness=-1

\begin{restatable}[\textbf{hessian alignment under IRM}]{corollary}{corhessianirm}\label{cor:hessian_alignment_irm}
Under the same setup as in \Cref{thm:moment_alignment_irm}, we have:
    \begin{equation}
        \mathrm{T}_\Gamma (\mathcal{S}\|\mathcal{T}) \leq \frac{1}{2}\delta \max_{i \neq j} \|\mathbf{H}_{\mu_j}\left(\boldsymbol{\theta}^*\right) -\mathbf{H}_{\mu_i} \left(\boldsymbol{\theta}^*\right)\|_F  + o(\delta)
    \end{equation}
where $\mathbf{H}(\boldsymbol{\theta})$ denotes the Hessian matrix of $\mathcal{L}(\boldsymbol{\theta})$.
\end{restatable}
\Cref{cor:hessian_alignment_irm} implies that the transfer measure is upper-bounded by the Frobenius norm of the difference of the Hessian matrices across domains.\ This result aligns with the findings of \citet{hemati_understanding_2023}.\ However, unlike their result, \Cref{cor:hessian_alignment_irm} does not require knowledge of the target domain; instead, it relies on the assumption of $\mathcal{T}$ being a convex combination of the source domains. 

\vspace{-0.25em}
\subsection{Moment Alignment without IRM Assumption}\label{sec:theory_no_irm}
\vspace{-0.25em}
In practice, unless the features are explicitly trained or post-processed to satisfy the IRM assumption—such as through IRM training or Invariant Feature Subspace Recovery~\citep{wang_invariant-feature_2023}—invariant optimal predictors generally do not exist. In this section, we derive a bound on the transfer measure under this setting.

\begin{assumption}[\textbf{bounded gradients, approximate IRM}]\label{assumption:bounded_gradients}
There exists a constant $g > 0$ such that
    $
        \min _{\boldsymbol{\theta} \in \mathcal{H}} \max _{i \in[K]} \|\nabla_{\boldsymbol{\theta}} \mathcal{L}_{\mu_i} \left(\boldsymbol{\theta}\right)\|_2 \leq g$.
\end{assumption}
\begin{restatable}[\textbf{moment alignment}]{theorem}{thmnonirm}
    \label{thm:moment_alignment} Given $K$ source domains $\mathcal{S} = \left\{\mu_i\right\}_{i=1}^K$ and target domain $\mathcal{T} \in conv(\mu_1, \dots, \mu_k)$.\ Assume loss $\mathcal{L}_{\mu_i}$, $\forall i \in [K]$ are $\nu$-strongly convex and $M$-times differentiable w.r.t. the classifier head. {Let $\mathcal{P}(\{\mathcal{L}_{\mu_i}\}_{i= 1}^K):=\{\boldsymbol{\theta}: \max_{i \in [K]} (\boldsymbol{\theta^{\prime}} - \boldsymbol{\theta})^{\top}\nabla_{\boldsymbol{\theta}}\mathcal{L}_{\mu_i}(\boldsymbol{\theta}) \geq 0, \; \forall \boldsymbol{\theta^{\prime}} \in \Gamma\}$ (a set of weakly Pareto optimal points for the objectives $\{\mathcal{L}_{\mu_i}\}_{i= 1}^K$), and let $\boldsymbol{\theta}^* \in  \argmin _{\boldsymbol{\theta} \in \mathcal{P}} \max_{i \in [K]} \|\nabla {\mathcal{L}}_{\mu_i} \left(\boldsymbol{\theta}\right)\|_2$},
    $\Gamma 
    := \left\{\boldsymbol{\theta} \in \mathcal{H}: \max_{i \in [K]} \left(\mathcal{L}_{\mu_i} \left(\boldsymbol{\theta}\right) - \mathcal{L}_{\mu_i} \left(\boldsymbol{\theta}^*\right)\right) \leq \delta_{\mathcal{S}}\right\}$, and $\delta = \frac{2\delta_\mathcal{S}}{\nu}$, we have:
    \begin{equation}
        \begin{aligned}
            \mathrm{T}_\Gamma (\mathcal{S} \| \mathcal{T})
             \leq & \frac{1}{2} \max_{i \neq j} \bigg(\mathcal{L}_{\mu_j}\left(\boldsymbol{\theta}^*\right)- \mathcal{L}_{\mu_j}(\boldsymbol{\theta}^*_j)\ \\
            - & \left(\mathcal{L}_{\mu_i} \left(\boldsymbol{\theta}^*\right) - \mathcal{L}_{\mu_i} \left(\boldsymbol{\theta}^*_i\right)\right)\\
            + & \sum_{n = 1}^N \frac{1}{n !} \delta^{\frac{n}{2}} \left\|\mathtt{\nabla^n_{\boldsymbol{\theta}} \mathcal{L}_{\mu_j} \left(\boldsymbol{\theta}^*\right) }- \mathtt{\nabla^n_{\boldsymbol{\theta}} \mathcal{L}_{\mu_i} \left(\boldsymbol{\theta}^*\right) }\right\|_F \bigg) \\
            + & o(\delta^{\frac{N}{2}}) 
        \end{aligned}
    \end{equation}
    where $\boldsymbol{\theta}^*_i$ is the minimizer of $\mathcal{L}_{\mu_i} \left(\boldsymbol{\theta}\right)$. Furthermore, suppose~\Cref{assumption:bounded_gradients} holds with $g > 0$: 
    \begin{equation}
        \begin{aligned}
            \mathrm{T}_\Gamma (\mathcal{S} \| \mathcal{T})
              \leq  & \delta^{\frac{1}{2}} g  + \frac{1}{2} \max_{i \neq j} \bigg(\mathcal{L}_{\mu_j}\left(\boldsymbol{\theta}^*\right)- \mathcal{L}_{\mu_j}\left(\boldsymbol{\theta}^*_j\right) \\ 
              - & \left(\mathcal{L}_{\mu_i} \left(\boldsymbol{\theta}^*\right) - \mathcal{L}_{\mu_i} \left(\boldsymbol{\theta}^*_i\right) \right)   \\
               + & \sum_{n = 2}^N \frac{1}{n !} \delta^{\frac{n}{2}} \left\|\mathtt{\nabla^n_{\boldsymbol{\theta}} \mathcal{L}_{\mu_j} \left(\boldsymbol{\theta}^*\right) }- \mathtt{\nabla^n_{\boldsymbol{\theta}} \mathcal{L}_{\mu_i} \left(\boldsymbol{\theta}^*\right) }\right\|_F\bigg) \\
               + & o(\delta^{\frac{N}{2}}) 
        \end{aligned}
    \end{equation}
\end{restatable}
As a special case, when $N = 2$, we have the following guarantee on gradient and Hessian alignment.

\begin{restatable}[\textbf{hessian alignment}]{corollary}{corhessiannoirm}\label{cor:hessian_alignment_non_irm}
Under the same setup as in \Cref{thm:moment_alignment}, we have:
    \begin{equation}
    \begin{aligned}
            \mathrm{T}_\Gamma (\mathcal{S}\|\mathcal{T}) \leq & \frac{1}{2} \max_{i \neq j}  \bigg(\mathcal{L}_{\mu_j}\left(\boldsymbol{\theta}^*\right)-\mathcal{L}_{\mu_j}\left(\boldsymbol{\theta}^*_j\right)\\ 
            - &\left(\mathcal{L}_{\mu_i} \left(\boldsymbol{\theta}^*\right)
            -  \mathcal{L}_{\mu_i} \left(\boldsymbol{\theta}^*_i\right) \right)\\
            + & \delta^{\frac{1}{2}} \left\|\nabla_{\boldsymbol{\theta}}\mathcal{L}_{\mu_j}\left(\boldsymbol{\theta}^*\right) - \nabla_{\boldsymbol{\theta}}\mathcal{L}_{\mu_j}\left(\boldsymbol{\theta}^*\right)\right\|_2  \\
            + & \frac{1}{2}\delta \left\|\mathbf{H}_{\mu_j}\left(\boldsymbol{\theta}^*\right) - \mathbf{H}_{\mu_i} \left(\boldsymbol{\theta}^*\right)\right\|_F \bigg)+ o(\delta)
    \end{aligned}    
    \end{equation}
    where $\boldsymbol{\theta}^*_i$ is the minimizer of $\mathcal{L}_{\mu_i} \left(\boldsymbol{\theta}\right)$. 

    Furthermore, suppose~\Cref{assumption:bounded_gradients} holds with $g > 0$: 
    \begin{equation}
    \begin{aligned}
        \mathrm{T}_\Gamma (\mathcal{S}\|\mathcal{T}) \leq 
        & \delta^{\frac{1}{2}} g + \frac{1}{2}  \max_{i \neq j} \bigg(
            \mathcal{L}_{\mu_j}\left(\boldsymbol{\theta}^*\right)
            - \mathcal{L}_{\mu_j}\left(\boldsymbol{\theta}^*_j\right) \\
            - & \left(\mathcal{L}_{\mu_i} \left(\boldsymbol{\theta}^*\right)
            - \mathcal{L}_{\mu_i} \left(\boldsymbol{\theta}^*_i\right)\right)  \\
            + & \frac{1}{2} \delta \left\|\mathbf{H}_{\mu_j}\left(\boldsymbol{\theta}^*\right)
            - \mathbf{H}_{\mu_i} \left(\boldsymbol{\theta}^*\right)\right\|_F
        \bigg) + o(\delta)
    \end{aligned}
\end{equation}

\end{restatable}

To summarize, under the IRM assumption, the transfer measure is bounded by the differences in higher-order derivatives (second order and above) across domains. Conversely, when the IRM assumption does not hold, the transfer measure is bounded by the maximum optimality gaps and the gradient norms, in addition to the differences in higher-order derivatives across domains.

The implication of \Cref{thm:moment_alignment} is the necessity of minimizing the gradient norm, as the upper bound depends on it regardless of the differences in higher-order derivatives. Fortunately, this bound can be reduced by incorporating gradient norm minimization—a strategy already embedded in many existing methods, as we will see later.

The results above rely on the assumption that the loss is strongly convex w.r.t. the classifier head, which is satisfied by widely used losses with L2 regularization, such as cross-entropy loss or mean-square error.

\vspace{-0.5em}
\section{Moment Alignment: A Unifying Framework}\label{sec:unif}
\vspace{-0.5em}
While various approaches to DG exist, they appear largely disconnected, and, to the best of our knowledge, no prior work has explicitly drawn connections between them. In this section, we unify IRM, gradient matching, and Hessian matching under the CMA framework. We further establish a duality between feature learning space and classifier fitting.
\vspace{-0.25em}
\subsection{IRM as Moment Alignment}\label{thm:unif_irm}
\vspace{-0.25em}
When the features are fixed and satisfy the IRM assumption, minimizing the IRMv1 objective~\citep{arjovsky_invariant_2020}
\begin{equation}\label{eq:irmv1_loss}
    \begin{aligned}
        \mathcal{L}_{\text{IRM}} := \mathcal{L}_{\text{ERM}} +  \lambda \frac{1}{K}\sum_{i = 1}^K \| \nabla_{\boldsymbol{\theta}} \mathcal{L}_{\mu_i}\left(\boldsymbol{\theta}\right) \|_2^2,
    \end{aligned}
    \tag{IRMv1}
\end{equation}
 recovers such invariant optimal predictor, and~\Cref{thm:moment_alignment_irm} provides an upper bound on the target error. On the other hand, when the fixed features do not satisfy the IRM assumption,
 The IRMv1 penalty seeks a parameter $\boldsymbol{\theta}$ whose average gradient norm is small, thereby minimizing $g$ in the upper bound in \Cref{thm:moment_alignment}.

 \vspace{-0.5em}
\subsection{Gradient and Hessian Matching as Moment Alignment}
\vspace{-0.5em}
Their general gradient and Hessian matching objectives are either the following or their variants:
\begin{equation}\label{eq:gm_loss}
\small
    \begin{aligned}
        \mathcal{L}_{\text{GM}} := \mathcal{L}_{\text{ERM}} +\lambda \frac{1}{K} \sum_{i = 1}^K \left\|\nabla_{\boldsymbol{\theta}} \mathcal{L}_{\mu_i} \left(\boldsymbol{\theta}\right)- \overline{\nabla_{\boldsymbol{\theta}} \mathcal{L}\left(\boldsymbol{\theta}\right)}\right\|^2_2
    \end{aligned}
    \tag{GM}
\end{equation}
\begin{equation}\label{eq:hm_loss} \tag{HM}
\small
    \mathcal{L}_{\text{HM}} := \mathcal{L}_{\text{ERM}} +  \lambda \frac{1}{K}\sum_{i = 1}^K \| \mathbf{H}_{\mu_i}\left(\boldsymbol{\theta}\right) - \overline{\mathbf{H}\left(\boldsymbol{\theta}\right)}\|_F^2
\end{equation}
By their definitions, gradient matching and Hessian matching are special cases of moment alignment, reducing the first-order and second-order terms, respectively, in the upper bound of the transfer measure. Notably, when the IRM assumption holds, the penalty in \Cref{eq:gm_loss} will favor an invariant optimal predictor.

From the results in \Cref{sec:theory}, aligning both gradients and Hessians improves DG over aligning only one of them. This explains the success of HGP and Hutchinson~\citep{hemati_understanding_2023} over methods that focus on gradient matching~\citep{shi_gradient_2021, parascandolo_learning_2020, koyama_out--distribution_2020} or Hessian matching~\citep{rame_fishr_2022, sun_deep_2016}.
\vspace{-0.5em}
\subsection{Feature Matching as Moment Alignment}\label{sec:unif_feature}
\vspace{-0.5em}
So far, we have discussed moment alignment under fixed features. Next, we establish a connection between the derivatives of the classifier and moments of features, where the classifier is assumed to be the last layer of an NN, i.e., linear predictor over the learned features.

For a softmax classifier, the prediction is a function of $\mathbf{x}^{\top} \boldsymbol{\theta} $, where $\mathbf{x}$ is a feature vector and $\boldsymbol{\theta}$ is the classifier. Therefore, $\mathtt{\nabla^n_{\boldsymbol{\theta}} \ell(\boldsymbol{\theta})} $ involves the $n^{\text{th}}$ moment of $\mathbf{x}$, and by matching the $n^{\text{th}}$ order derivatives w.r.t. the classifier head, we are matching the $n^{\text{th}}$ moment of $\mathbf{x}$ across domains.  Another view of this duality is that by the symmetry between $\mathbf{x}$ and  $\boldsymbol{\theta}$, we can derive analogously results in \Cref{sec:theory} with optimization target $\mathbf{x}$.

IRM~\citep{ahuja_invariant_2020} and CORAL~\citep{sun_deep_2016} are two concrete examples of this feature-parameter duality. Going from the feature space to the parameter space, CORAL \citep{sun_deep_2016} matches the feature covariance, namely the second moment of $\mathbf{x}$. Thus, CORAL is approximately Hessian matching in the parameter space. We refer interested readers to Proposition 4 in \citet{hemati_understanding_2023} for discussion on the attributes aligned by CORAL. Conversely, starting from the parameter space and moving to the feature space, the penalty term in \Cref{eq:irmv1_loss} regularizes the gradient w.r.t. the classifier, corresponding to the first-moment alignment in the feature space, i.e., aligning the features themselves.

\vspace{-0.5em}
\section{Closed-Form Moment Alignment}\label{sec:cma_algo}
\vspace{-0.5em}
Motivated by the theory of moment alignment, we introduce \textbf{C}losed-Form \textbf{M}oment \textbf{A}lignment (CMA), a DG algorithm that minimizes the following objective: \begin{equation}\label{eq:cma_loss} \tag{CMA}
\begin{aligned}
    \mathcal{L}_{\text{CMA}} = &\frac{1}{K}\sum_{i=1}^K \mathcal{L}_{\mu_i} + \alpha \| \nabla_{\boldsymbol{\theta}} \mathcal{L}_{\mu_i} - \overline{\nabla_{\boldsymbol{\theta}} \mathcal{L}}\|_2^2 \\
    + &\beta \| \mathbf{H}_{\mu_i} - \overline{\mathbf{H}}\|_F^2,
\end{aligned}
\end{equation} where $\overline{\nabla_{\boldsymbol{\theta}} \mathcal{L}} = \frac{1}{K} \sum_{i=1}^K \nabla_{\boldsymbol{\theta}} \mathcal{L}_{\mu_i}$ and $\overline{\mathbf{H}} = \frac{1}{K} \sum_{i=1}^K \mathbf{H}_{\mu_i}$ are the average gradient and Hessian. Similar to HGP and Hutchinson~\citep{hemati_understanding_2023}, CMA aligns the gradients and Hessians across domains, but we compute the derivatives in closed form. In \Cref{app:cma_conection}, we connect CMA and other DG algorithms.

Gradient and Hessian matching~\citep{koyama_out--distribution_2020, shi_gradient_2021,hemati_understanding_2023, rame_fishr_2022}, despite their theoretical and empirical success, often incur significant computations due to multiple backpropagations for a single update. CMA bypasses this limitation by analytically computing gradient and Hessian penalty.

\vspace{-1em}
\subsection{Closed-Form Gradient and Hessian}
\vspace{-1em}
CMA computes the gradient and Hessian penalty without requiring additional backpropagations.\ Leveraging closed-form solutions for the gradients and Hessians of the cross-entropy loss w.r.t. a linear classifier, CMA reduces computational overhead.\ The derivations are provided in \Cref{app:ce_grad_hess_derivation}.
\vspace{-1em}
\subsection{Memory-Efficient Hessian Matching}
\vspace{-1em}
The Hessian of the cross-entropy loss for a single feature vector $\mathbf{x}$ w.r.t. a softmax classifier is:
\begin{equation*}
    \mathbf{H} = \left( \mathrm{diag}(\mathbf{p}) - \mathbf{p} \mathbf{p}^\top \right) \otimes \left( \mathbf{x} \mathbf{x}^\top \right),
\end{equation*}
where $\mathbf{p} \in \mathbb{R}^{C}$ is a vector of predicted probabilities, $\mathrm{diag}(\mathbf{p}) \in \mathbb{R}^{C \times C}$ is the diagonal matrix with elements of $\mathbf{p}$, $\mathbf{x} \mathbf{x}^\top \in \mathbb{R}^{d \times d}$, and $\otimes$ is the Kronecker product. 

The dimension of $\mathbf{H}$ is quadratic in the number of classes $C$ and feature dimension $d$, which could be memory-prohibitive under many features or classes. To mitigate this issue, we use properties of the Frobenius norm to avoid storing the full $dC \times dC$ matrix. Instead, we compute:
\begin{equation}\label{hession_norm}
    \|\mathbf{H}\|^2_F = 
    \operatorname{tr}\left( \mathrm{diag}(\mathbf{p}) - \mathbf{p} \mathbf{p}^\top \right) \operatorname{tr}\left( \mathbf{x} \mathbf{x}^\top \right)
\end{equation}
which only requires storing a $C \times C$ and a $d \times d$ matrix. 

\vspace{-1em}
\subsection{Trade-offs in Hessian Computation}
\vspace{-1em}

Our code offers two versions of Hessian computation, each with its trade-offs.\ The first version directly computes the Frobenius norm of the Kronecker product, which is faster but requires more memory.\ The second version avoids storing the full Kronecker product matrix, reducing memory usage, but requires computing traces for all pairs of Hessians.\ We lay out the derivation in \Cref{app:eff_hessian}.
\vspace{-0.5em}

\begin{table*}[t]
    \centering
    \caption{Test accuracy (\%) with standard error over three datasets. Each experiment is repeated over 5 seeds.}
\vspace{-0.5em}
    \centering
    \resizebox{0.8\textwidth}{!}{%
        \begin{tabular}{@{}lcccccc@{}}
            \toprule
            \textbf{Method} & \multicolumn{2}{c}{\textbf{Waterbirds (CLIP ViT-B/32)}} & \multicolumn{2}{c}{\textbf{CelebA (CLIP ViT-B/32)}} & \multicolumn{2}{c}{\textbf{MultiNLI (BERT)}}                                                                         \\
            \cmidrule(l){2-3} \cmidrule(l){4-5} \cmidrule(l){6-7}
                            & \textbf{Average}                                        & \textbf{Worst-Group}                                & \textbf{Average}                             & \textbf{Worst-Group}  & \textbf{Average}      & \textbf{Worst-Group}  \\
            \midrule
            ERM & 89.52 ± 0.10 & 84.58 ± 0.20 & 78.76 ± 0.03 & 72.22 ± 0.39 & 81.15 ± 0.30 & 68.82 ± 0.64 \\ 
            CORAL & 89.67 ± 0.14 & 84.85 ± 0.22 & \textbf{78.81 ± 0.03} & 73.00 ± 0.22 & 81.22 ± 0.21 & 68.71 ± 0.52 \\
            Fishr & 89.79 ± 0.10 & 86.08 ± 0.10 & 73.95 ± 0.86 & 69.63 ± 1.20 & \textbf{81.35 ± 0.16} & \textbf{71.55 ± 1.20} \\  
            CMA & \textbf{90.11 ± 0.17} & \textbf{86.16 ± 0.29} & 77.87 ± 0.04 & \textbf{74.16 ± 0.10} & 81.30 ± 0.25 & 69.72 ± 0.66 \\
        \bottomrule
    \end{tabular}\label{tab:real-datasets-transposed}}
\vspace{-1em}
\end{table*}

\begin{table*}[t]
    \centering
    \caption{DomainBed results with test-domain validation model selection.}
    \vspace{-0.5em}
    \centering
    \resizebox{0.8\textwidth}{!}{%
\begin{tabular}{lcccccc}
\toprule
\textbf{Algorithm}        & \textbf{ColoredMNIST}     & \textbf{RotatedMNIST}     & \textbf{VLCS}             & \textbf{PACS}             & \textbf{TerraIncognita}   & \textbf{Avg}              \\
\midrule
ERM                       & 54.5 $\pm$ 0.2            & 97.8 $\pm$ 0.1            & 76.9 $\pm$ 0.3            & 80.2 $\pm$ 0.5            & 36.5 $\pm$ 0.5            & 69.2                      \\
CORAL                     & 55.7 $\pm$ 0.5            & \textbf{98.0 $\pm$ 0.0}            & 75.9 $\pm$ 0.2            & 80.2 $\pm$ 0.2            & 33.6 $\pm$ 0.5            & 68.7                      \\
Fishr                     & 62.0 $\pm$ 1.7            & 97.9 $\pm$ 0.0            & \textbf{77.5 $\pm$ 0.5}            & 81.5 $\pm$ 0.2            & 37.3 $\pm$ 1.1            & 71.2                      \\
CMA          & \textbf{62.5 $\pm$ 0.9}            & 97.9 $\pm$ 0.1            & 77.4 $\pm$ 0.8            & \textbf{81.6 $\pm$ 0.3}            & \textbf{38.4 $\pm$ 1.2}            & \textbf{71.5}                      \\
\bottomrule
\end{tabular}
}
\label{tab:domainbed}
\vspace{-1em}
\end{table*}

\vspace{-0.5em}
\section{Experiments}\label{sec:experiment}
\vspace{-0.5em}
We validate CMA through both quantitative and qualitative analyses. First, we describe the experimental setup, including dataset details and model training procedures. We then present quantitative results, evaluating CMA’s performance under the \hyperref[assumption:IRM]{IRM assumption} and scenarios where it does not hold. Finally, we conduct qualitative analyses to better understand CMA’s effect on worst-group performance and feature moment alignment.

\vspace{-1em}
\subsection{Implementation}\label{sec:implementation}
\vspace{-1em}
\textbf{Linear Probing (IRM)}~~We evaluate liner probing performance on Waterbirds~\citep{sagawa_distributionally_2020}, CelebA~\citep{liu_deep_2015}, and MultiNLI~\citep{williams_broad-coverage_2018}. To enforce the \hyperref[assumption:IRM]{IRM assumption}, we apply the Invariant-feature Subspace Recovery (ISR) algorithm~\citep{wang_provable_2022, wang_invariant-feature_2023}, which provably yields features that induce an optimal invariant predictor. For Waterbirds and CelebA, we extract features from a CLIP-pretrained Vision Transformer (ViT-B/32). For MultiNLI, we fine-tune a BERT model using the code and hyperparameters in \citet{sagawa_distributionally_2020}, then extract features from the fine-tuned model. These features are transformed using the ISR-mean algorithm~\citep{wang_provable_2022,wang_invariant-feature_2023}. Finally, we train a linear classifier using ERM, Fishr~\citep{rame_fishr_2022}, and CMA objectives.

\textbf{Full Fine-Tuning (Non-IRM)}~~We run end-to-end fine-tuning on a subset of DomainBed~\citep{gulrajani_search_2020}, applying gradient and Hessian regularization from \Cref{eq:cma_loss} to the classifier head while back-propagating the loss through both the linear classifier and the encoder. Specifically, penalizing large gradient variance aligns gradients across domains, while the ERM loss drives gradients toward zero. The two mechanisms promote a small gradient norm for each domain, aligning with the theory in \Cref{sec:theory_no_irm} and \Cref{sec:unif_feature}. Given recent empirical evidence supporting strong DG capabilities of Vision Transformers~\citep{ghosal_are_2022, zheng_prompt_2022, sultana_self-distilled_2022}, we have selected ViT-S as the backbone for DomainBed experiments. Using the DomainBed codebase~\citep{gulrajani_search_2020}, we compare ERM~\citep{vapnik_overview_1999}, CORAL~\citep{sun_deep_2016}, Fishr~\citep{rame_fishr_2022}, and CMA by fine-tuning small Vision Transformers~\citep{steiner_how_2022, dosovitskiy_image_2021, wightman_pytorch_2019}. For more implementation details, please refer to \Cref{app:irm_exp} and \Cref{app:non_irm_exp}.
\vspace{-1em}
\subsection{Quantitative Results}
\vspace{-1em}

Our goal is not to claim that CMA surpassed existing algorithms but to demonstrate that our framework encompasses gradient matching (e.g., \citet{koyama_out--distribution_2020}) and Hessian matching (e.g., \citet{sun_deep_2016,rame_fishr_2022,hemati_understanding_2023}). To this end, our experimental results confirm that CMA achieves performance comparable to state-of-the-art moment matching methods.

\textbf{Linear Probing (IRM)} From \Cref{tab:real-datasets-transposed}, we observe that CMA consistently outperforms ERM on worst-group accuracy while maintaining comparable average accuracy across all datasets. Compared to Fishr, CMA achieves higher worst-group and average accuracy on two out of three datasets. In contrast, Fishr's performance varies, underperforming ERM on CelebA. Compared to CORAL, CMA achieves better worst-group performance across all datasets, while maintaining similar average accuracy. 

\textbf{Full Fine-Tuning (Non-IRM)} We follow \citet{rame_fishr_2022} to employ the test-domain model selection method, where the validation set is a holdout set from the test domain. As shown in \Cref{tab:domainbed}, CMA achieves comparable performance to Fishr, with both methods consistently outperforming ERM. This result supports the performance guarantee in \Cref{cor:hessian_alignment_non_irm} and validates our unified framework. Please refer to \Cref{app:per_data_result} for per-dataset and training-domain validation performance. 
\begin{figure}[t]
    \centering
    \includegraphics[width=\linewidth]{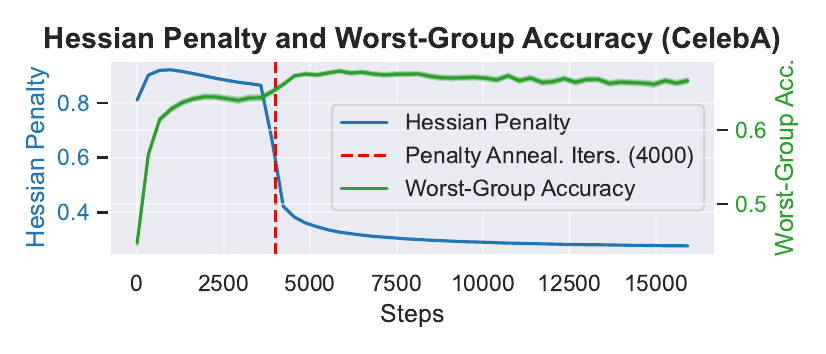}
    \vspace{-2.5em}
    \caption{Hessian Penalty and worst-case accuracy on CelebA. Both curves represent the mean values, with shaded areas indicating $\pm$ one standard deviation over five runs.}
    \label{fig:hess_acc}
    \vspace{-1em}
\end{figure}
\vspace{-1em}
\subsection{Qualitative Results}\label{sec:exp_qual}
\vspace{-1em}
\textbf{Effect of Hessian Matching} 
We analyze CMA’s training progression and its impact on worst-group performance by plotting the Hessian loss:  
\begin{equation*}
    \frac{\beta}{K} \sum_{i = 1}^K \|\mathbf{H}_{\mu_i}(\boldsymbol{\theta}) - \overline{\mathbf{H}(\boldsymbol{\theta})}\|_F^2
\end{equation*}
for linear probing on the CelebA dataset, with the same hyperparameters as those reported for accuracy in \Cref{tab:real-datasets-transposed} ($\alpha = 5000$, $\beta = 100$, $\text{penalty annealing iterations} = 4000$).
As shown in \Cref{fig:hess_acc}, near step 4000, when the gradient and Hessian matching terms take effect, there is a sharp drop in Hessian penalty, accompanied by a noticeable increase in worst-case accuracy, aligning with our theory that aligning Hessians across domains improves worst-case performance.

\begin{figure}[t]
    \centering
    \includegraphics[width=0.9\linewidth]{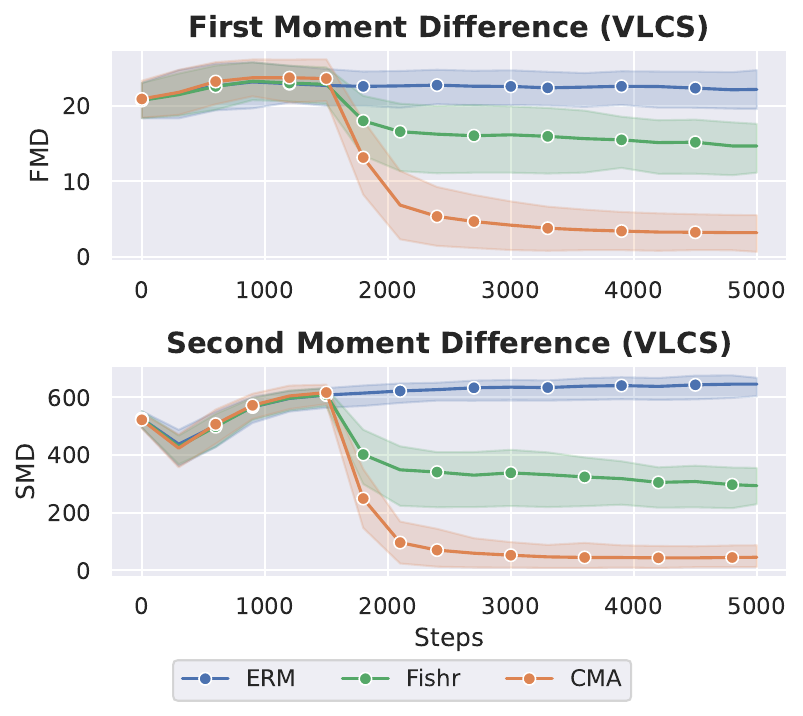}
    \caption{Comparison of first and second-moment differences across test domains for the VLCS dataset. The plots show the progression of moment differences over training steps for ERM, Fishr, and CMA. ERM fails to align the feature moments while CMA achieves the most effective alignment. The shaded regions represent one standard deviation above and below the mean across test domains.}
    \label{fig:cmnist_mom}
\vspace{-1em}
\end{figure}

\textbf{Feature Moment Matching}~~As discussed in \Cref{sec:unif_feature}, we illustrate the effect of CMA in matching the moments of features across domains. \Cref{fig:cmnist_mom} presents the moment differences between domains on VLCS dataset, where we average over all test domains. While ERM shows significant discrepancies in feature moments between domains, both Fishr and CMA successfully reduce these differences. Notably, CMA is more effective in reducing both first and second-moment disparities.
\vspace{-1em}
\subsection{Runtime and Memory Comparison}\label{app:time_mem}
\vspace{-1em}
We report the average time per step (in seconds) and memory usage (in GB) for each (algorithm, dataset) pair in  \Cref{tab:wallclock} and \Cref{tab:mem}. It is important to note that, in addition to the algorithms' efficiency, the wall-clock time also depends on the hardware status at the time of training. We include additional comparisons of two versions of CMA, HGP, and Hutchinson algorithms, where ``CMA (Speed)'' uses the time-efficient Hessian computation, while ``CMA (Memory)'' uses the memory-efficient Hessian computation.

Among the methods compared, only CMA and Hutchinson compute full Hessian matrices. While CMA is inherently slower than Fishr, CORAL, and HGP, which rely on diagonal approximations of the Hessian, it remains more time-efficient than Hutchinson's.

To highlight the scalability of ``CMA (Memory)'', we run small-scale experiments on OfficeHome, a dataset with 65 classes. In this setting, ``CMA (Speed)'' requires more than 75\,GB of memory and could not run on a single GPU, whereas ``CMA (Memory)'' completed successfully with peak usage under 13.7\,GB.

\begin{table*}[ht]
\centering
\caption{Wall-clock time across datasets (in seconds). Algorithms are grouped by the type of moment matching. For each dataset, we bold the most time-efficient algorithm within each category.}
\resizebox{\textwidth}{!}{%
\begin{tabular}{lcccccc}
\toprule
\textbf{Algorithm} & \textbf{ColoredMNIST (2)} & \textbf{RotatedMNIST (10)} & \textbf{VLCS (5)} & \textbf{PACS (7)} & \textbf{TerraIncognita (10)} & \textbf{OfficeHome (65)} \\
\midrule
\multicolumn{7}{l}{\textit{No Moment Matching}} \\
\hspace{1em}ERM          & 0.0278 & 0.0403 & 0.4019 & 0.3620 & 0.4216 & 0.4064 \\

\midrule
\multicolumn{7}{l}{\textit{Approximate Second-Order}} \\
\hspace{1em}CORAL        & \textbf{0.0457} & \textbf{0.1003} & 0.6241 & \textbf{0.5244} & 0.7697 & 0.5279 \\
\hspace{1em}Fishr        & 0.0925 & 0.1331 & 0.7472 & 0.6757 & 0.6057 & 0.6600 \\
\hspace{1em}HGP          & 0.0657 & 0.1292 & \textbf{0.6048} & 0.6729 & \textbf{0.6045} & \textbf{0.4977} \\

\midrule
\multicolumn{7}{l}{\textit{Exact Second-Order}} \\
\hspace{1em}Hutchinson   & 4.1663 & 9.7935 & 7.7604 & 7.3284 & 7.7270 & 7.8446 \\
\hspace{1em}CMA (Speed)  & \textbf{0.0676} & \textbf{0.1326} & \textbf{0.7354} & \textbf{0.7266} & \textbf{0.7421} & --     \\
\hspace{1em}CMA (Memory) & 0.1226 & 0.4723 & 0.8874 & 1.1699 & 1.0685 & \textbf{0.8495} \\
\bottomrule
\end{tabular}\label{tab:wallclock}
}
\end{table*}

\begin{table*}[ht]
\centering
\caption{Memory usage across datasets (in GB). For each dataset, we bold the most memory-efficient algorithm within each category.}
\resizebox{\textwidth}{!}{%
\begin{tabular}{lcccccc}
\toprule
\textbf{Algorithm} & \textbf{ColoredMNIST (2)} & \textbf{RotatedMNIST (10)} & \textbf{VLCS (5)} & \textbf{PACS (7)} & \textbf{TerraIncognita (10)} & \textbf{OfficeHome (65)} \\
\midrule
\multicolumn{7}{l}{\textit{No Moment Matching}} \\
\hspace{1em}ERM          & 0.1550 & 0.3728 & 6.8865 & 6.8865 & 6.8865 & 6.8868 \\

\midrule
\multicolumn{7}{l}{\textit{Approximate Second-Order}} \\
\hspace{1em}CORAL        & \textbf{0.1391} & 0.3190 & 6.7093 & 6.7093 & 6.7093 & 6.7097 \\
\hspace{1em}Fishr        & 0.3192 & 0.7936 & 14.1433 & 14.1436 & 14.1441 & 14.1522 \\
\hspace{1em}HGP          & 0.1477 & \textbf{0.3099} & \textbf{5.6835} & \textbf{5.6835} & \textbf{5.6836} & \textbf{5.6843} \\
\midrule
\multicolumn{7}{l}{\textit{Exact Second-Order}} \\
\hspace{1em}Hutchinson   & \textbf{0.1502} & \textbf{0.3496} & \textbf{5.7047} & \textbf{5.7125} & \textbf{5.7284} & \textbf{6.0029} \\
\hspace{1em}CMA (Speed)  & 0.2867 & 1.4537 & 13.9511 & 14.8391 & 16.7272 & \textasciitilde75 \\
\hspace{1em}CMA (Memory) & 0.3914 & 0.7776 & 13.6447 & 13.6448 & 13.6448 & 13.6474 \\
\bottomrule
\end{tabular}\label{tab:mem}
}
\end{table*}

\vspace{-1em}
\section{Limitations and Future Directions}\label{sec:limitation}
\vspace{-1em}
Our analysis assumes that the target distribution is in the convex hull of the source distributions, which may not always hold or be verifiable in practice. It might be of future interest to relax this convexity assumption to accommodate a broader range of target distributions.

While closed-form Hessians eliminate the need for sampling-based approximations or multiple backpropagations, they introduce scalability challenges. Specifically, the Hessian matrix of the cross-entropy loss w.r.t. the classifier head scales quadratically with the number of classes and feature dimensions. To remedy this challenge, we provide a memory-efficient alternative for computing the Hessian Frobenius norm, albeit at the cost of longer computation time. Future work could explore Hessian approximations to further balance efficiency and accuracy.

The primary focus of this work is on a unifying theory for DG using gradient and Hessian matching. Our theory suggests that aligning higher-order derivatives improves generalization, but in practice, even second-order alignment is computationally demanding. The feasibility and potential benefits of higher-order alignments remain open questions, presenting intriguing directions for future research.

\vspace{-0.5em}
\section{Conclusions}\label{sec:conclusions}
\vspace{-0.5em}
We introduced a unified theory of moment alignment for DG, providing upper bounds by examining differences in derivatives. The moment alignment framework reinterprets IRM, gradient matching, and Hessian matching, explaining the success of algorithms that match both gradients and Hessians across domains. Additionally, we established the duality between moments of features and derivatives of classifier heads, a novel perspective that we believe will open new research avenues.

Inspired by our theory, we proposed Closed-Form Moment Alignment (CMA), a DG algorithm that aligns gradients and Hessians analytically, avoiding the computational inefficiencies of previous methods that relied on repeated backpropagation or sampling-based Hessian estimations. We validated the efficacy of CMA through both quantitative and qualitative experiments. The results demonstrated that CMA achieves performance on par with state-of-the-art methods (e.g., Fishr). These findings not only confirm our theoretical predictions but also underscore the practical benefits of our moment alignment framework.

\vspace{-0.5em}
\begin{acknowledgements} 
\vspace{-0.5em}
    This work is partially supported by an NSF IIS grant with No.\ 2416897. HZ would like to thank the support from a Google Research Scholar Award. This research used the Delta advanced computing and data resource which is supported by the National Science Foundation (award OAC 2005572) and the State of Illinois. We thank Haoxiang Wang for suggestions of our experiments. The views and conclusions expressed in this paper are solely those of the authors and do not necessarily reflect the official policies or positions of the supporting companies and government agencies.
\end{acknowledgements}

\newpage
\bibliographystyle{plainnat}
\bibliography{references}

\onecolumn

\title{Moment Alignment: Unifying Gradient and Hessian Matching for Domain Generalization\\(Supplementary Material)}
\maketitle
\appendix
\section{Proof of Propositions}\label{app:proof_prop}
\propuppertm*
\begin{proof}
    \begin{equation}\label{eq:one_sided_upperbound1}
        \begin{aligned}
            \mathrm{T}_{\Gamma}(\mathcal{S} \| \mathcal{T}) & = \mathrm{T}_{\Gamma}\left(\mu^* \| \mathcal{T}\right)                                                                                                         \\
                                                            & =\mathrm{T}_{\Gamma}\left(\mu^* \middle\| \sum_{i=1}^K w_i \mu_i\right)                                                                                        \\
                                                            & =\sup _{h \in \Gamma} \sum_{i=1}^K w_i\mathcal{L}_{\mu_i}(h)-\mathcal{L}_\mathcal{T}^*-\left(\mathcal{L}_{\mu^*}(h)-\mathcal{L}_{\mu^*}^*\right)               \\
                                                            & \leq \sup _{h \in \Gamma} \sum_{i=1}^K w_i\left[\mathcal{L}_{\mu_i}(h)-\mathcal{L}_\mathcal{T}^*-\left(\mathcal{L}_{\mu^*}(h)-\mathcal{L}_{\mu^*}^*\right) \right] \\
                                                            & \leq \sup _{h \in \Gamma} \sum_{i=1}^K w_i\left[\mathcal{L}_{\mu_i}(h)-\mathcal{L}_{\mu_i}^*-\left(\mathcal{L}_{\mu^*}(h)-\mathcal{L}_{\mu^*}^*\right) \right] \\
                                                            & \leq \sum_{i=1}^K w_i \sup _{h \in \Gamma}\left[\mathcal{L}_{\mu_i}(h)-\mathcal{L}_{\mu_i^*}-\left(\mathcal{L}_{\mu^*}(h)-\mathcal{L}_{\mu^*}^*\right) \right] \\
                                                            & =\sum_{i=1}^K w_i \mathrm{T}_{\Gamma}\left(\mu^*\| \mu_i\right)                                                                                                \\
                                                            & \leq \max _{i \in[k]} \mathrm{T}_{\Gamma}\left(\mu^* \| \mu_i\right)
        \end{aligned}
    \end{equation}
    On the other hand, let $j_{max}:= \argmax _{j \in [K]}  \mathrm{T}_{\Gamma}\left(\mu_j \| \mu^*\right)$ and for any fixed $i$, let $\mu_{mid}$ be such that $ \mathrm{T}_{\Gamma}\left(\mu_{j_{max}} \| \mu_{i}\right) = \mathrm{T}_{\Gamma}\left(\mu_{j_{max}} \| \mu_{mid}\right) + \mathrm{T}_{\Gamma}\left(\mu_{mid} \|\mu_{i} \right)$. By the definition $\mu^*$ we have:
    \begin{equation}
        \begin{aligned}
            \frac{1}{2} \max _{i \neq j} \mathrm{T}_{\Gamma}\left(\mu_{j} \| \mu_i\right)
            \geq & \frac{1}{2} \mathrm{T}_{\Gamma}\left(\mu_{j_{max}} \| \mu_{i}\right) \quad \forall i \in [K]          \\
            =    & \mathrm{T}_{\Gamma}\left(\mu_{mid} \|\mu_{i} \right) \quad \forall i \in [K], \quad \exists \mu_{mid} \\
            \geq & \mathrm{T}_{\Gamma}\left(\mu^* \| \mu_{i}\right) \quad \forall i \in [K]                              \\
        \end{aligned}
        \label{eq:one_sided_upperbound2}
    \end{equation}
    Combine \Cref{eq:one_sided_upperbound1} and \Cref{eq:one_sided_upperbound2}, we have:
    \begin{equation}
        \mathrm{T}_{\Gamma}(\mathcal{S} \| \mathcal{T}) \leq \max _{i \in[k]} \mathrm{T}_{\Gamma}\left(\mu^* \| \mu_i\right) \leq \frac{1}{2}\max _{i \neq j} \mathrm{T}_{\Gamma}\left(\mu_j \| \mu_i\right) \qedhere
    \end{equation}
\end{proof}
\proptemulti*
\begin{proof}
    Apply \Cref{prop:target_error_bound} on $\mu^*$ and $\mathcal{T}$, we have:
    \begin{equation}
        \begin{aligned}
            \mathcal{L}_{\mathcal{T}}(h) \leq & \mathcal{L}_{\mu^*}(h)+\mathcal{L}_{\mathcal{T}}^*-\mathcal{L}_{\mu^*}^*+\mathrm{T}_{\Gamma}(\mu^* \| \mathcal{T})                                              \\
            =                                 & \mathcal{L}_{\mathcal{S}}(h)+\mathcal{L}_{\mathcal{T}}^*-\mathcal{L}_{\mathcal{S}}^*+\mathrm{T}_{\Gamma}(\mathcal{S} \| \mathcal{T})                            \\
            \leq                              & \mathcal{L}_{\mathcal{S}}(h)+\mathcal{L}_{\mathcal{T}}^*-\mathcal{L}_{\mathcal{S}}^*+\frac{1}{2}\max _{i \neq j} \mathrm{T}_{\Gamma}\left(\mu_j \| \mu_i\right)
        \end{aligned}
    \end{equation}
    The second line follows from the definition of $\mu^*$ and the last line follows from \Cref{prop:tm_half}.
\end{proof}

\section{Proof of Theorems}\label{app:proof_thm}
\thmirm*
\begin{proof}\label{app:proof_thm_irm}
    From the $\nu$-strong convexity of $\mathcal{L}_{\mu_i} \left(\boldsymbol{\theta}\right)$, we can write for all $i$
    \begin{equation}\label{strong_convex}
        \begin{aligned}
                     & \mathcal{L}_{\mu_i} \left(\boldsymbol{\theta}\right)
            \geq \mathcal{L}_{\mu_i} \left(\boldsymbol{\theta}^*\right) + \underbrace{(\boldsymbol{\theta} - \boldsymbol{\theta}^*)^\top \nabla_{\boldsymbol{\theta}} \mathcal{L}_{\mu_i} \left(\boldsymbol{\theta}^*\right)}_{=0} + \frac{\nu}{2}\left\|\boldsymbol{\theta} - \boldsymbol{\theta}^*\right\|^2_2 \\
            \implies & \mathcal{L}_{\mu_i} \left(\boldsymbol{\theta}\right)
            - \mathcal{L}_{\mu_i} \left(\boldsymbol{\theta}^*\right) \geq \frac{\nu}{2}\left\|\boldsymbol{\theta} - \boldsymbol{\theta}^*\right\|^2_2
        \end{aligned}
    \end{equation}
    So for any $\boldsymbol{\theta} \in \Gamma$
    \begin{equation}
        \frac{\nu}{2}\left\|\boldsymbol{\theta} - \boldsymbol{\theta}^*\right\|^2_2 \leq
        \max_{i \in [n]} \left(\mathcal{L}_{\mu_i} \left(\boldsymbol{\theta}\right) - \mathcal{L}_{\mu_i} \left(\boldsymbol{\theta}^*\right)\right)\leq
        \delta_{\mathcal{S}}
        \label{f1_f2_1}
    \end{equation}
    If we define set $\mathcal{F}_2$ as
    \begin{equation}
        \mathcal{F}_2 = \left\{\boldsymbol{\theta}: \frac{\nu}{2}\left\|\boldsymbol{\theta} - \boldsymbol{\theta}^*\right\|^2_2 \leq \delta_{\mathcal{S}} \right\}
    \end{equation}
    then \Cref{f1_f2_1} implies $\Gamma \subset \mathcal{F}_2$.

    From \Cref{prop:tm_half}, we have
    \begin{equation}
        \begin{aligned}
            \mathrm{T}_\Gamma (\mathcal{S}\|\mathcal{T}) \leq & \frac{1}{2} \max _{i \neq j} \mathrm{T}_{\Gamma}\left(\mu_{j} \| \mu_i\right)                                                                                                                                                                                         \\
            =                                                 & \frac{1}{2} \max _{i \neq j} \sup _{\boldsymbol{\theta} \in \Gamma} \left(\mathcal{L}_{\mu_j}\left(\boldsymbol{\theta}\right)-\mathcal{L}_{\mu_j}\left(\boldsymbol{\theta}^*\right)-\left(\mathcal{L}_{\mu_i} \left(\boldsymbol{\theta}\right)-\mathcal{L}_{\mu_i} \left(\boldsymbol{\theta}^*\right)\right) \right)                     \\
            =                                                 & \frac{1}{2} \max _{i \neq j} \sup _{\left\|\boldsymbol{\theta} - \boldsymbol{\theta}^*\right\|^2_2 \leq \delta} \left(\mathcal{L}_{\mu_j}\left(\boldsymbol{\theta}\right)-\mathcal{L}_{\mu_j}\left(\boldsymbol{\theta}^*\right)-\left(\mathcal{L}_{\mu_i} \left(\boldsymbol{\theta}\right)-\mathcal{L}_{\mu_i} \left(\boldsymbol{\theta}^*\right)\right) \right) \\
        \end{aligned}
    \end{equation}
    To bound the terms inside the supremum, which is the difference in excess risk of $\mu_j$ and $\mu_i$, we write the Taylor expansion of $\mathcal{L}_{\mu_i} \left(\boldsymbol{\theta}\right)$ around $\boldsymbol{\theta}^*$:
    \begin{equation}
        \begin{aligned}
            & \mathcal{L}_{\mu_i} \left(\boldsymbol{\theta}\right)
              = \mathcal{L}_{\mu_i} \left(\boldsymbol{\theta}^*\right) + \underbrace{(\boldsymbol{\theta} - \boldsymbol{\theta}^*)^\top \nabla_{\boldsymbol{\theta}} \mathcal{L}_{\mu_i} \left(\boldsymbol{\theta}^*\right)}_{=0} +  \sum_{n = 2}^N \frac{1}{n !} \mathtt{\nabla^n_{\boldsymbol{\theta}} \mathcal{L}_{\mu_i} \left(\boldsymbol{\theta}^*\right) }\mathtt{\mathtt{\left(\boldsymbol{\theta} - \boldsymbol{\theta}^*\right)^{\otimes n}}} + o(\left\|\boldsymbol{\theta} - \boldsymbol{\theta}^*\right\|^N_2) \\
            \implies &\mathcal{L}_{\mu_i} \left(\boldsymbol{\theta}\right)
            - \mathcal{L}_{\mu_i} \left(\boldsymbol{\theta}^*\right)
              =   \sum_{n = 2}^N \frac{1}{n !} \mathtt{\nabla^n_{\boldsymbol{\theta}} \mathcal{L}_{\mu_i} \left(\boldsymbol{\theta}^*\right) }\mathtt{\left(\boldsymbol{\theta} - \boldsymbol{\theta}^*\right)^{\otimes n}}+ o(\left\|\boldsymbol{\theta} - \boldsymbol{\theta}^*\right\|^N_2)
        \end{aligned}
        \label{taylor:mu_i}
    \end{equation}
    Here ${\otimes n}$ denote the $n^{\text{th}}$-order tensor product, where $\mathtt{\left( \boldsymbol{\theta} -\boldsymbol{\theta}^*\right)^{\otimes n}}_{(k_1, \cdots, k_n)}$ is the product of $ \left( \boldsymbol{\theta}_{k_1} -\boldsymbol{\theta}^*_{k_1}\right), \dots, \left( \boldsymbol{\theta}_{k_n} -\boldsymbol{\theta}^*_{k_n}\right)$.
    
    Similarly, expanding $\mathcal{L}_{\mu_j}\left(\boldsymbol{\theta}\right)$ around $\boldsymbol{\theta}^*$, we have:
    \begin{equation}
        \begin{aligned}
            \mathcal{L}_{\mu_j}\left(\boldsymbol{\theta}\right)
            - \mathcal{L}_{\mu_j}\left(\boldsymbol{\theta}^*\right)
             & =   \sum_{n = 2}^N \frac{1}{n !} \mathtt{\nabla^n_{\boldsymbol{\theta}}  \mathcal{L}_{\mu_j} \left(\boldsymbol{\theta}^*\right)} \mathtt{\left(\boldsymbol{\theta} - \boldsymbol{\theta}^*\right)^{\otimes n}} + o(\left\|\boldsymbol{\theta} - \boldsymbol{\theta}^*\right\|^N_2)
        \end{aligned}
    \end{equation}
    These two equations together give an upper bound on the difference in excess risk of domain $j$ and $i$:
    \begin{equation}
        \begin{aligned}
                 & \mathcal{L}_{\mu_j}\left(\boldsymbol{\theta}\right) - \mathcal{L}_{\mu_j}\left(\boldsymbol{\theta}^*\right) - \left(\mathcal{L}_{\mu_i} \left(\boldsymbol{\theta}\right) - \mathcal{L}_{\mu_i} \left(\boldsymbol{\theta}^*\right)\right)                                           \\
            =    &  \sum_{n = 2}^N \frac{1}{n !} \left(\mathtt{\nabla^n_{\boldsymbol{\theta}} \mathcal{L}_{\mu_j} \left(\boldsymbol{\theta}^*\right) }- \mathtt{\nabla^n_{\boldsymbol{\theta}} \mathcal{L}_{\mu_i} \left(\boldsymbol{\theta}^*\right)}\right) \mathtt{\left(\boldsymbol{\theta} - \boldsymbol{\theta}^*\right)^{\otimes n}} + o(\left\|\boldsymbol{\theta} - \boldsymbol{\theta}^*\right\|^N_2) \\
            \leq & \sum_{n = 2}^N \frac{1}{n !} \left\|\mathtt{\nabla^n_{\boldsymbol{\theta}} \mathcal{L}_{\mu_j} \left(\boldsymbol{\theta}^*\right) }- \mathtt{\nabla^n_{\boldsymbol{\theta}} \mathcal{L}_{\mu_i} \left(\boldsymbol{\theta}^*\right)}\right\|_F \left\| \mathtt{\left(\boldsymbol{\theta} - \boldsymbol{\theta}^*\right)^{\otimes n}}\right\|_F + o(\left\|\boldsymbol{\theta} - \boldsymbol{\theta}^*\right\|^N_2)  \\
            \leq & \sum_{n = 2}^N \frac{1}{n !} \left\|\mathtt{\nabla^n_{\boldsymbol{\theta}} \mathcal{L}_{\mu_j} \left(\boldsymbol{\theta}^*\right) }- \mathtt{\nabla^n_{\boldsymbol{\theta}} \mathcal{L}_{\mu_i} \left(\boldsymbol{\theta}^*\right)}\right\|_F \left\|\boldsymbol{\theta} - \boldsymbol{\theta}^*\right\|_2^n + o(\left\|\boldsymbol{\theta} - \boldsymbol{\theta}^*\right\|^N_2)                                                                  \\
        \end{aligned}
    \end{equation}
    Taking the supremum over $\boldsymbol{\theta} \in \mathcal{F}_2$ on both sides, for any $i \neq j$, we have:
    \begin{equation}
        \begin{aligned}
             &\sup_{\left\|\boldsymbol{\theta}-\boldsymbol{\theta}^*\right\|^2_2 \leq \delta} \left(\mathcal{L}_{\mu_j}\left(\boldsymbol{\theta}\right) - \mathcal{L}_{\mu_j}\left(\boldsymbol{\theta}^*\right) - \left(\mathcal{L}_{\mu_i} \left(\boldsymbol{\theta}\right) - \mathcal{L}_{\mu_i} \left(\boldsymbol{\theta}^*\right)\right)\right)\\
            \leq & \sum_{n = 2}^N \frac{1}{n !} \delta^\frac{n}{2} \left\|\mathtt{\nabla^n_{\boldsymbol{\theta}} \mathcal{L}_{\mu_j} \left(\boldsymbol{\theta}^*\right) }- \mathtt{\nabla^n_{\boldsymbol{\theta}} \mathcal{L}_{\mu_i} \left(\boldsymbol{\theta}^*\right)}\right\|_F  + o(\delta^\frac{N}{2})
        \end{aligned}
    \end{equation}

    Finally, by taking the maximum over $i$ and $j$, $i \neq j$, on both sides, we can bound the transfer measure as follows:
    \begin{equation}
        \begin{aligned}
            \mathrm{T}_\Gamma (\mathcal{S}\|\mathcal{T}) \leq & \frac{1}{2} \max _{i \neq j} \sup_{\left\|\boldsymbol{\theta}-\boldsymbol{\theta}^*\right\|^2_2 \leq \delta} \left(\mathcal{L}_{\mu_j}\left(\boldsymbol{\theta}\right) - \mathcal{L}_{\mu_j}\left(\boldsymbol{\theta}^*\right) - \left(\mathcal{L}_{\mu_i} \left(\boldsymbol{\theta}\right) - \mathcal{L}_{\mu_i} \left(\boldsymbol{\theta}^*\right)\right)\right) \\
            \leq  & \max_{i \neq j} \sum_{n = 2}^N \frac{1}{n !} \delta^\frac{n}{2} \left\|\mathtt{\nabla^n_{\boldsymbol{\theta}} \mathcal{L}_{\mu_j} \left(\boldsymbol{\theta}^*\right) }- \mathtt{\nabla^n_{\boldsymbol{\theta}} \mathcal{L}_{\mu_i} \left(\boldsymbol{\theta}^*\right)}\right\|_F  + o(\delta^\frac{N}{2})
        \end{aligned} \qedhere
    \end{equation}
\end{proof}

\thmnonirm*
\begin{proof}\label{app:proof_thm_nonirm}
    From the $\nu$-strong convexity of $\mathcal{L}_{\mu_i} \left(\boldsymbol{\theta}\right)$, we can write for all $i$
    \begin{equation}
        \begin{aligned}
                     & \mathcal{L}_{\mu_i} \left(\boldsymbol{\theta}\right)
            \geq \mathcal{L}_{\mu_i} \left(\boldsymbol{\theta}^*\right) + {(\boldsymbol{\theta} - \boldsymbol{\theta}^*)^\top \nabla_{\boldsymbol{\theta}} \mathcal{L}_{\mu_i} \left(\boldsymbol{\theta}^*\right)} + \frac{\nu}{2}\|\boldsymbol{\theta} - \boldsymbol{\theta}^*\|^2_2 \\
\implies & {\delta_{\mathcal{S}} \geq  \max_{i \in [K]} \mathcal{L}_{\mu_i} \left(\boldsymbol{\theta}\right) -\mathcal{L}_{\mu_i} \left(\boldsymbol{\theta^*}\right)\geq
             \underbrace{\max_{i \in [K]} (\boldsymbol{\theta} - \boldsymbol{\theta}^*)^\top \nabla_{\boldsymbol{\theta}} \mathcal{L}_{\mu_i} \left(\boldsymbol{\theta}^*\right)}_{\geq 0 \text{ since $\boldsymbol{\theta^*}$ is weak Pareto Optimal}} + \frac{\nu}{2}\|\boldsymbol{\theta} - \boldsymbol{\theta}^*\|^2_2 \geq \frac{\nu}{2}\|\boldsymbol{\theta} - \boldsymbol{\theta}^*\|^2_2}
            \label{strong_convex_40}
        \end{aligned}
    \end{equation}
    Now define $\mathcal{F}_2$ as
    \begin{equation}
        \mathcal{F}_2 = \left\{\boldsymbol{\theta}: \max_{i \in [K]} \frac{\nu}{2}\|\boldsymbol{\theta} - \boldsymbol{\theta}^*\|^2_2 \leq \delta_{\mathcal{S}} \right\} = \left\{\boldsymbol{\theta}: \max_{i \in [K]} \|\boldsymbol{\theta} - \boldsymbol{\theta}^*\|^2_2 \leq \frac{2 \delta_{\mathcal{S}}}{\nu} = \delta \right\}
    \end{equation}
    From \Cref{strong_convex_40}, we have $\Gamma \subseteq \mathcal{F}_2$, and thus $\mathrm{T}_\Gamma (\mathcal{S}\|\mathcal{T}) \leq \mathrm{T}_{\mathcal{F}_2} (\mathcal{S}\|\mathcal{T})$.\\
    \begin{equation}
        \begin{aligned}
            \mathrm{T}_{\mathcal{F}_2} (\mathcal{S}\|\mathcal{T}) \leq & \frac{1}{2} \max _{i \neq j} \mathrm{T}_{\mathcal{F}_2}\left(\mu_{j} \| \mu_i\right)                                                                                                                                                                                 \\
            =                                                          & \frac{1}{2} \max _{i \neq j} \sup _{\boldsymbol{\theta} \in \mathcal{F}_2} \left({\mathcal{L}}_{\mu_j}\left(\boldsymbol{\theta}\right)-{\mathcal{L}}_{\mu_j}\left(\boldsymbol{\theta}^*_j\right)-\left({\mathcal{L}}_{\mu_i} \left(\boldsymbol{\theta}\right)-{\mathcal{L}}_{\mu_i} \left(\boldsymbol{\theta}^*_i\right)\right) \right)
        \end{aligned}
    \end{equation}
    We write the Taylor expansion of ${\mathcal{L}}_{\mu_i} \left(\boldsymbol{\theta}\right)$ and ${\mathcal{L}}_{\mu_j}\left(\boldsymbol{\theta}\right)$  around $\boldsymbol{\theta}^*$ as done in the proof of \Cref{thm:moment_alignment_irm}:
    \begin{equation*}
        \begin{aligned}
            \mathcal{L}_{\mu_i} \left(\boldsymbol{\theta}\right)
             & = \mathcal{L}_{\mu_i} \left(\boldsymbol{\theta}^*\right) + (\boldsymbol{\theta} - \boldsymbol{\theta}^*)^\top \nabla_{\boldsymbol{\theta}} \mathcal{L}_{\mu_i} \left(\boldsymbol{\theta}^*\right) +  \sum_{n = 2}^N \frac{1}{n !} \mathtt{\nabla^n_{\boldsymbol{\theta}} \mathcal{L}_{\mu_i} \left(\boldsymbol{\theta}^*\right) }\mathtt{\left(\boldsymbol{\theta} - \boldsymbol{\theta}^*\right)^{\otimes n}}  + o(\left\|\boldsymbol{\theta} - \boldsymbol{\theta}^*\right\|^N_2)
        \end{aligned}
    \end{equation*}
    \begin{equation*}
        \begin{aligned}
            \mathcal{L}_{\mu_j}\left(\boldsymbol{\theta}\right)
             & = \mathcal{L}_{\mu_j}\left(\boldsymbol{\theta}^*\right) + (\boldsymbol{\theta} - \boldsymbol{\theta}^*)^\top \nabla_{\boldsymbol{\theta}} \mathcal{L}_{\mu_j}\left(\boldsymbol{\theta}^*\right) +  \sum_{n = 2}^N \frac{1}{n !} \mathtt{\nabla^n_{\boldsymbol{\theta}} \mathcal{L}_{\mu_j} \left(\boldsymbol{\theta}^*\right) }\mathtt{\left(\boldsymbol{\theta} - \boldsymbol{\theta}^*\right)^{\otimes n}}  + o(\left\|\boldsymbol{\theta} - \boldsymbol{\theta}^*\right\|^N_2)
        \end{aligned}
    \end{equation*}

    Combining the two equations above, we have:
    \begin{equation*}
        \begin{aligned}
                 & {\mathcal{L}}_{\mu_j}\left(\boldsymbol{\theta}\right)-{\mathcal{L}}_{\mu_j}\left(\boldsymbol{\theta}^*_j\right)-\left({\mathcal{L}}_{\mu_i} \left(\boldsymbol{\theta}\right)-{\mathcal{L}}_{\mu_i} \left(\boldsymbol{\theta}^*_i\right) \right) \\
                 \leq & \mathcal{L}_{\mu_j}\left(\boldsymbol{\theta}^*\right) + (\boldsymbol{\theta}- \boldsymbol{\theta}^*)^\top \nabla_{\boldsymbol{\theta}} \mathcal{L}_{\mu_j}\left(\boldsymbol{\theta}^*\right) +  \sum_{n = 2}^N \frac{1}{n !} \mathtt{\nabla^n_{\boldsymbol{\theta}} \mathcal{L}_{\mu_j} \left(\boldsymbol{\theta}^*\right) }\mathtt{\left(\boldsymbol{\theta} - \boldsymbol{\theta}^*\right)^{\otimes n}} -{\mathcal{L}}_{\mu_i} \left(\boldsymbol{\theta}^*_j\right) \\
                 & - \left( \mathcal{L}_{\mu_i} \left(\boldsymbol{\theta}^*\right) + (\boldsymbol{\theta} - \boldsymbol{\theta}^*)^\top \nabla_{\boldsymbol{\theta}} \mathcal{L}_{\mu_i} \left(\boldsymbol{\theta}^*\right) +  \sum_{n = 2}^N \frac{1}{n !} \mathtt{\nabla^n_{\boldsymbol{\theta}} \mathcal{L}_{\mu_i} \left(\boldsymbol{\theta}^*\right) }\mathtt{\left(\boldsymbol{\theta} - \boldsymbol{\theta}^*\right)^{\otimes n}} -{\mathcal{L}}_{\mu_i} \left(\boldsymbol{\theta}^*_i\right)\right)  + o(\left\|\boldsymbol{\theta} - \boldsymbol{\theta}^*\right\|^N_2) \\
            \leq & \mathcal{L}_{\mu_j}\left(\boldsymbol{\theta}^*\right)- \mathcal{L}_{\mu_j}\left(\boldsymbol{\theta}^*_j\right) - \left( \mathcal{L}_{\mu_i} \left(\boldsymbol{\theta}^*\right) - \mathcal{L}_{\mu_i} \left(\boldsymbol{\theta}^*_i\right) \right) \\
                 & + \sum_{n =1}^N \frac{1}{n !} \left\|\mathtt{\nabla^n_{\boldsymbol{\theta}} \mathcal{L}_{\mu_j} \left(\boldsymbol{\theta}^*\right) }- \mathtt{\nabla^n_{\boldsymbol{\theta}} \mathcal{L}_{\mu_i} \left(\boldsymbol{\theta}^*\right) }\right\|_F \left\|\boldsymbol{\theta} - \boldsymbol{\theta}^*\right\|^{n}_2  + o(\left\|\boldsymbol{\theta} - \boldsymbol{\theta}^*\right\|^N_2) 
        \end{aligned}
    \end{equation*}

    Taking the supremum over $\boldsymbol{\theta} \in \mathcal{F}_2$ on both sides, for any $i \neq j$, we have:
    \begin{equation*}
        \begin{aligned}
                 & \sup_{\boldsymbol{\theta} \in \mathcal{F}_2}  \mathcal{L}_{\mu_j}\left(\boldsymbol{\theta}\right)- \mathcal{L}_{\mu_j}\left(\boldsymbol{\theta}^*_j\right) - \left(\mathcal{L}_{\mu_i} \left(\boldsymbol{\theta}\right) - \mathcal{L}_{\mu_i} \left(\boldsymbol{\theta}^*_i\right)\right)                                             
             + \sum_{n = 1}^N \frac{1}{n !} \left\|\mathtt{\nabla^n_{\boldsymbol{\theta}} \mathcal{L}_{\mu_j} \left(\boldsymbol{\theta}^*\right) }- \mathtt{\nabla^n_{\boldsymbol{\theta}} \mathcal{L}_{\mu_i} \left(\boldsymbol{\theta}^*\right) }\right\|_F \left\|\boldsymbol{\theta} - \boldsymbol{\theta}^*\right\|^{n}_2  + o(\left\|\boldsymbol{\theta} - \boldsymbol{\theta}^*\right\|^N_2)                                                                                                                                                                                        \\
            \leq & \mathcal{L}_{\mu_j}\left(\boldsymbol{\theta}^*\right)- \mathcal{L}_{\mu_j}\left(\boldsymbol{\theta}^*_j\right) + \mathcal{L}_{\mu_i} \left(\boldsymbol{\theta}^*\right) - \mathcal{L}_{\mu_i} \left(\boldsymbol{\theta}^*_i\right) 
            + \sum_{n = 1}^N \frac{1}{n !} \delta^{\frac{n}{2}} \left\|\mathtt{\nabla^n_{\boldsymbol{\theta}} \mathcal{L}_{\mu_j} \left(\boldsymbol{\theta}^*\right) }- \mathtt{\nabla^n_{\boldsymbol{\theta}} \mathcal{L}_{\mu_i} \left(\boldsymbol{\theta}^*\right) }\right\|_F +  o(\delta^{\frac{N}{2}}) 
        \end{aligned}
    \end{equation*}
    Finally, maximizing over $i \neq j$ on both sides, the transfer measure is bounded by:
    \begin{equation}
        \begin{aligned}
            \mathrm{T}_\Gamma (\mathcal{S}\|\mathcal{T}) \leq & \frac{1}{2} \max _{i \neq j} \sup_{\boldsymbol{\theta} \in \mathcal{F}_2} \left(\mathcal{L}_{\mu_j}\left(\boldsymbol{\theta}\right) - \mathcal{L}_{\mu_j}\left(\boldsymbol{\theta}^*\right) - \left(\mathcal{L}_{\mu_i} \left(\boldsymbol{\theta}\right) - \mathcal{L}_{\mu_i} \left(\boldsymbol{\theta}^*\right)\right)\right) \\
            \leq                                              & \frac{1}{2} \max_{i \neq j} \mathcal{L}_{\mu_j}\left(\boldsymbol{\theta}^*\right)- \mathcal{L}_{\mu_j}\left(\boldsymbol{\theta}^*_j\right) + \mathcal{L}_{\mu_i} \left(\boldsymbol{\theta}^*\right) - \mathcal{L}_{\mu_i} \left(\boldsymbol{\theta}^*_i\right) 
            \\
                                                              & \quad \quad \quad + \sum_{n = 2}^N \frac{1}{n !} \delta^{\frac{n}{2}} \left\|\mathtt{\nabla^n_{\boldsymbol{\theta}} \mathcal{L}_{\mu_j} \left(\boldsymbol{\theta}^*\right) }- \mathtt{\nabla^n_{\boldsymbol{\theta}} \mathcal{L}_{\mu_i} \left(\boldsymbol{\theta}^*\right) }\right\|_F +  o(\delta^{\frac{N}{2}})
        \end{aligned}
        \label{eq:general_bound}
    \end{equation}
    We have proved the first part of the theorem.
    Now suppose there exists a constant $g$ such that $\min _{\boldsymbol{\theta} \in \mathcal{H}} \max _{i \in[K]} \|\nabla_{\boldsymbol{\theta}} \mathcal{L}_{\mu_i} \left(\boldsymbol{\theta}\right)\|_2 \leq g$, we can further upper bound the first-order term ($n=1$) by $g$:
        \begin{equation}
        \begin{aligned}
         \delta^{\frac{1}{2}} \left\|\nabla_{\boldsymbol{\theta}} \mathcal{L}_{\mu_j} \left(\boldsymbol{\theta}^*\right) - \nabla_{\boldsymbol{\theta}} \mathcal{L}_{\mu_i} \left(\boldsymbol{\theta}^*\right) \right\|_F
         = & \delta^{\frac{1}{2}} \left\|\nabla_{\boldsymbol{\theta}} \mathcal{L}_{\mu_j} \left(\boldsymbol{\theta}^*\right) - \nabla_{\boldsymbol{\theta}} \mathcal{L}_{\mu_i} \left(\boldsymbol{\theta}^*\right) \right\|_2 \\
         \leq &\delta^{\frac{1}{2}} \left\|\nabla_{\boldsymbol{\theta}} \mathcal{L}_{\mu_j} \left(\boldsymbol{\theta}^*\right)\right\|_2 + \left\|\nabla_{\boldsymbol{\theta}} \mathcal{L}_{\mu_i} \left(\boldsymbol{\theta}^*\right) \right\|_2 \\
         \leq &2 \delta^{\frac{1}{2}} g
        \end{aligned}
    \end{equation}
Replacing the first-order term in \Cref{eq:general_bound} with this upper bound completes the proof.
\end{proof}
\begin{restatable}[weakly Pareto optimal~\citep{chang_chapter_2015}]{definition}{wpo}
   A point $\boldsymbol{\theta} \in \Gamma$ is weakly Pareto optimal iff $\nexists$ another point $\boldsymbol{\theta}^\prime \in \Gamma$ such that $\mathcal{L}_{\mu_i} (\boldsymbol{\theta}^\prime)<\mathcal{L}_{\mu_i} \left(\boldsymbol{\theta}\right)\;\; \forall i$.
   \end{restatable}

\begin{restatable}[]{lemma}{lemWPO}
For convex $\{\mathcal{L}_{\mu_i}\}_{i= 1}^K$, $\boldsymbol{\theta} \in \mathcal{P}(\{\mathcal{L}_{\mu_i}\}_{i= 1}^K)$ iff  $\boldsymbol{\theta}$ is weakly Pareto optimal.
\end{restatable}
\begin{proof}
$(\implies)$ Let $\boldsymbol{\theta} \in \mathcal{P}(\{\mathcal{L}_{\mu_i}\}_{i= 1}^K)$, by convexity, we have for all $\boldsymbol{\theta}^\prime \in \Gamma $ and $i \in [K]$, $$\mathcal{L}_{\mu_i}(\boldsymbol{\theta}^\prime) \geq \mathcal{L}_{\mu_i}(\boldsymbol{\theta}) + (\boldsymbol{\theta}^\prime - \boldsymbol{\theta})^\top \nabla_{\boldsymbol{\theta}}\mathcal{L}_{\mu_i}(\boldsymbol{\theta}) \geq \mathcal{L}_{\mu_i}(\boldsymbol{\theta}),$$ where the gradient term is non-negative by $\boldsymbol{\theta} \in \mathcal{P}(\{\mathcal{L}_{\mu_i}\}_{i= 1}^K)$. Thus, $\boldsymbol{\theta}$ is weakly Pareto optimal as for all $\boldsymbol{\theta}^\prime \in \Gamma$ there is some $i \in [K]$ such that $\mathcal{L}_{\mu_i}(\boldsymbol{\theta}^\prime) \geq \mathcal{L}_{\mu_i}(\boldsymbol{\theta})$.~\\
$(\impliedby)$ Suppose for contradiction that $\boldsymbol{\theta} \notin \mathcal{P}(\{\mathcal{L}_{\mu_i}\}_{i= 1}^K)$, Then there exists some $\boldsymbol{\theta}^\prime$ such that $(\boldsymbol{\theta}^\prime - \boldsymbol{\theta})^\top \nabla_{\boldsymbol{\theta}}\mathcal{L}_{\mu_i}(\boldsymbol{\theta}) < 0$ for all $i \in [K]$. Using the Taylor expansion,
$$\mathcal{L}_{\mu_i}(\boldsymbol{\theta}^\prime) = \mathcal{L}_{\mu_i}(\boldsymbol{\theta})  +  (\boldsymbol{\theta}^\prime - \boldsymbol{\theta})^\top \nabla_{\boldsymbol{\theta}}\mathcal{L}_{\mu_i}(\boldsymbol{\theta}) + \mathcal{O} \left(\|\boldsymbol{\theta}^\prime - \boldsymbol{\theta}\|^2_2\right).$$
Choosing a scaled step $\boldsymbol{\theta}^* = \boldsymbol{\theta} + \delta(\boldsymbol{\theta}^\prime -\boldsymbol{\theta})$ to still satisfy $(\boldsymbol{\theta}^* - \boldsymbol{\theta})^\top \nabla_{\boldsymbol{\theta}}\mathcal{L}_{\mu_i}(\boldsymbol{\theta}) < 0$. For small enough $\delta \rightarrow 0$, $\mathcal{O} \left(\|\boldsymbol{\theta}^* - \boldsymbol{\theta}\|^2_2\right)$ becomes negligible and  $$\mathcal{L}_{\mu_i}(\boldsymbol{\theta}^*) \rightarrow \mathcal{L}_{\mu_i}(\boldsymbol{\theta})  +  (\boldsymbol{\theta}^* - \boldsymbol{\theta})^\top \nabla_{\boldsymbol{\theta}}\mathcal{L}_{\mu_i}(\boldsymbol{\theta}) < \mathcal{L}_{\mu_i}(\boldsymbol{\theta}),$$ and $\boldsymbol{\theta}$ is not weakly Pareto optimal. By contrapositive, we have shown that if $\boldsymbol{\theta}$ is weakly Pareto optimal, $\boldsymbol{\theta} \in \mathcal{P}(\{\mathcal{L}_{\mu_i}\}_{i= 1}^K)$.
\end{proof}

Note that the convexity assumption is necessary, as we can construct a simple one-dimensional counterexample where $\theta \in \mathcal{P}(\{\mathcal{L}_{\mu_i}\}_{i= 1}^K)$ but not weakly Pareto optimal. We consider two functions $\mathcal{L}_1(\theta)$ and $\mathcal{L}_2(\theta)$ given by:
\begin{align*}
    \mathcal{L}_1(\theta) = \sin\theta, \;
    \mathcal{L}_2(\theta)= \theta^3, \; \Gamma = \mathbb{R}.
\end{align*}
At $\theta = 0$, we compute the gradients:
\begin{align*}
    \mathcal{L}_1'(0) &= \cos(0) = 1, \\
    \mathcal{L}_2'(0) &= 0.
\end{align*}
Thus, for any $\theta' \in \mathbb{R}$:
\begin{align*}
    (\theta' - 0) \mathcal{L}_1'(0) &= \theta', \\
    (\theta' - 0) \mathcal{L}_2'(0) &= 0.
\end{align*}
Taking the maximum, we have:
\[
\max \{\theta', 0\} \geq 0, \quad \forall \theta'.
\]
therefore $\theta = 0 \in \mathcal{P}(\mathcal{L}_1, \mathcal{L}_2)$.

However, $\theta = 0$ is not weakly Pareto optimal. Consider $\theta^* = -0.5$. The function values at $\theta^*$ are:
\begin{align*}
    \mathcal{L}_1(\theta^*) &= \sin(-0.5) = -\frac{\sqrt{2}}{2} < \mathcal{L}_1(0) = 0, \\
    \mathcal{L}_2(\theta^*) &= (-0.5)^3 = -0.125 < \mathcal{L}_2(0) = 0.
\end{align*}
Since there exists $\theta^*$ such that both $\mathcal{L}_i(\theta^*) < \mathcal{L}_i(\theta)$ for all $i$, $\theta = 0$ is not weakly Pareto optimal.
\section{Proof of Corollaries}\label{app:proof_coro}
\corhessianirm*
\begin{proof}\label{app:proof_coro_irm}
    We use the same $\mathcal{F}_2$ as in the proof of \Cref{thm:moment_alignment_irm} and have $\Gamma \subset \mathcal{F}_2$.

    From \Cref{prop:tm_half}, we have
    \begin{equation}
        \begin{aligned}
            \mathrm{T}_\Gamma (\mathcal{S}\|\mathcal{T}) \leq & \frac{1}{2} \max _{i \neq j} \mathrm{T}_{\Gamma}\left(\mu_{j} \| \mu_i\right)                                                                                                                                                                                         \\
            =                                                 & \frac{1}{2} \max _{i \neq j} \sup _{\boldsymbol{\theta} \in \Gamma} \left(\mathcal{L}_{\mu_j}\left(\boldsymbol{\theta}\right)-\mathcal{L}_{\mu_j}\left(\boldsymbol{\theta}^*\right)-\left(\mathcal{L}_{\mu_i} \left(\boldsymbol{\theta}\right)-\mathcal{L}_{\mu_i} \left(\boldsymbol{\theta}^*\right)\right) \right)                     \\
            =                                                 & \frac{1}{2} \max _{i \neq j} \sup _{\|\boldsymbol{\theta} - \boldsymbol{\theta}^*\|^2_2 \leq \delta} \left(\mathcal{L}_{\mu_j}\left(\boldsymbol{\theta}\right)-\mathcal{L}_{\mu_j}\left(\boldsymbol{\theta}^*\right)-\left(\mathcal{L}_{\mu_i} \left(\boldsymbol{\theta}\right)-\mathcal{L}_{\mu_i} \left(\boldsymbol{\theta}^*\right)\right) \right) \\
        \end{aligned}
    \end{equation}
    To bound the terms inside the supremum, which is the difference in excess risk of $\mu_j$ and $\mu_i$, we write the Taylor expansion of $\mathcal{L}_{\mu_i} \left(\boldsymbol{\theta}\right)$ around $\boldsymbol{\theta}^*$:
    \begin{equation}
        \begin{aligned}
            & \mathcal{L}_{\mu_i} \left(\boldsymbol{\theta}\right)
             = \mathcal{L}_{\mu_i} \left(\boldsymbol{\theta}^*\right) + \underbrace{(\boldsymbol{\theta} - \boldsymbol{\theta}^*)^\top \nabla_{\boldsymbol{\theta}} \mathcal{L}_{\mu_i} \left(\boldsymbol{\theta}^*\right)}_{=0} + \frac{1}{2}(\boldsymbol{\theta} - \boldsymbol{\theta}^*)^\top\mathbf{H}_{\mu_i} \left(\boldsymbol{\theta}^*\right)(\boldsymbol{\theta} - \boldsymbol{\theta}^*)  + o(\|\boldsymbol{\theta} - \boldsymbol{\theta}^*\|^2_2) \\
            \implies &\mathcal{L}_{\mu_i} \left(\boldsymbol{\theta}\right)
            - \mathcal{L}_{\mu_i} \left(\boldsymbol{\theta}^*\right)
             = \frac{1}{2}(\boldsymbol{\theta} - \boldsymbol{\theta}^*)^\top \mathbf{H}_{\mu_i} \left(\boldsymbol{\theta}^*\right)(\boldsymbol{\theta} - \boldsymbol{\theta}^*) + o(\|\boldsymbol{\theta} - \boldsymbol{\theta}^*\|^2_2)
        \end{aligned}
    \end{equation}
    Similarly, expand $\mathcal{L}_{\mu_j}\left(\boldsymbol{\theta}\right)$ around $\boldsymbol{\theta}^*$, we have:
    \begin{equation}
        \begin{aligned}
            \mathcal{L}_{\mu_j}\left(\boldsymbol{\theta}\right) - \mathcal{L}_{\mu_j}\left(\boldsymbol{\theta}^*\right)
             & = \frac{1}{2}(\boldsymbol{\theta} - \boldsymbol{\theta}^*)^\top \mathbf{H}_{\mu_j}\left(\boldsymbol{\theta}^*\right)(\boldsymbol{\theta} - \boldsymbol{\theta}^*) + o(\|\boldsymbol{\theta} - \boldsymbol{\theta}^*\|^2_2)
        \end{aligned}
    \end{equation}
    These two equations together give an upper bound on the difference in excess risk of domain $j$ and $i$:
    \begin{equation}
        \begin{aligned}
                 & \mathcal{L}_{\mu_j}\left(\boldsymbol{\theta}\right) - \mathcal{L}_{\mu_j}\left(\boldsymbol{\theta}^*\right) - \left(\mathcal{L}_{\mu_i} \left(\boldsymbol{\theta}\right) - \mathcal{L}_{\mu_i} \left(\boldsymbol{\theta}^*\right)\right)                                           \\
            =    & \frac{1}{2}(\boldsymbol{\theta} - \boldsymbol{\theta}^*)^\top\mathbf{H}_{\mu_j}\left(\boldsymbol{\theta}^*\right)(\boldsymbol{\theta} - \boldsymbol{\theta}^*) - \frac{1}{2}(\boldsymbol{\theta} - \boldsymbol{\theta}^*)^\top\mathbf{H}_{\mu_i} \left(\boldsymbol{\theta}^*\right)(\boldsymbol{\theta} - \boldsymbol{\theta}^*) + o(\|\boldsymbol{\theta} - \boldsymbol{\theta}^*\|^2_2) \\
            =    & \frac{1}{2}(\boldsymbol{\theta} - \boldsymbol{\theta}^*)^\top \left(\mathbf{H}_{\mu_j}\left(\boldsymbol{\theta}^*\right) -\mathbf{H}_{\mu_i} \left(\boldsymbol{\theta}^*\right)\right) (\boldsymbol{\theta} - \boldsymbol{\theta}^*) +o(\|\boldsymbol{\theta} - \boldsymbol{\theta}^*\|^2_2)                                          \\
            \leq & \frac{1}{2}\|\boldsymbol{\theta} - \boldsymbol{\theta}^*\|^2_2 \|\mathbf{H}_{\mu_j}\left(\boldsymbol{\theta}^*\right) -\mathbf{H}_{\mu_i} \left(\boldsymbol{\theta}^*\right)\|_2  + o(\|\boldsymbol{\theta} - \boldsymbol{\theta}^*\|^2_2)                                                                  \\
        \end{aligned}
    \end{equation}
    Taking the supremum over $\boldsymbol{\theta} \in \mathcal{F}_2$ on both sides, for any $i \neq j$, we have:
    \begin{equation*}
        \begin{aligned}
             & \sup_{\|\boldsymbol{\theta}-\boldsymbol{\theta}^*\|^2_2 \leq \delta} \left(\mathcal{L}_{\mu_j}\left(\boldsymbol{\theta}\right) - \mathcal{L}_{\mu_j}\left(\boldsymbol{\theta}^*\right) - \left(\mathcal{L}_{\mu_i} \left(\boldsymbol{\theta}\right) - \mathcal{L}_{\mu_i} \left(\boldsymbol{\theta}^*\right)\right)\right)
            \leq \frac{1}{2}\delta \|\mathbf{H}_{\mu_j}\left(\boldsymbol{\theta}^*\right) -\mathbf{H}_{\mu_i} \left(\boldsymbol{\theta}^*\right)\|_2  + o(\delta)
        \end{aligned}
    \end{equation*}
    Finally, by taking the maximum over $i$ and $j$, $i \neq j$, on both sides, we can bound the transfer measure as follows:
    \begin{equation}
        \begin{aligned}
            \mathrm{T}_\Gamma (\mathcal{S}\|\mathcal{T}) \leq & \frac{1}{2} \max _{i \neq j} \sup_{\|\boldsymbol{\theta}-\boldsymbol{\theta}^*\|^2_2 \leq \delta} \left(\mathcal{L}_{\mu_j}\left(\boldsymbol{\theta}\right) - \mathcal{L}_{\mu_j}\left(\boldsymbol{\theta}^*\right) - \left(\mathcal{L}_{\mu_i} \left(\boldsymbol{\theta}\right) - \mathcal{L}_{\mu_i} \left(\boldsymbol{\theta}^*\right)\right)\right) \\
            \leq                                              & \frac{1}{2}\delta \max_{i \neq j} \|\mathbf{H}_{\mu_j}\left(\boldsymbol{\theta}^*\right) -\mathbf{H}_{\mu_i} \left(\boldsymbol{\theta}^*\right)\|_2  + o(\delta)
        \end{aligned} \qedhere
    \end{equation}
\end{proof}

\corhessiannoirm*
\begin{proof}\label{app:proof_coro_nonirm}
    Use the same $\mathcal{F}_2$ as in the proof of \Cref{thm:moment_alignment}, we have $\Gamma \subseteq \mathcal{F}_2$, and $\mathrm{T}_\Gamma (\mathcal{S}\|\mathcal{T}) \leq \mathrm{T}_{\mathcal{F}_2} (\mathcal{S}\|\mathcal{T})$.
    \begin{equation}
        \begin{aligned}
            \mathrm{T}_{\mathcal{F}_2} (\mathcal{S}\|\mathcal{T}) \leq & \frac{1}{2} \max _{i \neq j} \mathrm{T}_{\mathcal{F}_2}\left(\mu_{j} \| \mu_i\right)                                                     \\
            =                                                          & \frac{1}{2} \max _{i \neq j} \sup _{\boldsymbol{\theta} \in \mathcal{F}_2} \left({\mathcal{L}}_{\mu_j}\left(\boldsymbol{\theta}\right)-{\mathcal{L}}_{\mu_j}\left(\boldsymbol{\theta}^*_j\right)-\left({\mathcal{L}}_{\mu_i} \left(\boldsymbol{\theta}\right)-{\mathcal{L}}_{\mu_i} \left(\boldsymbol{\theta}^*_i\right)\right) \right)
        \end{aligned}
    \end{equation}
    We write the Taylor expansion of ${\mathcal{L}}_{\mu_i} \left(\boldsymbol{\theta}\right)$ and ${\mathcal{L}}_{\mu_j}\left(\boldsymbol{\theta}\right)$  around $\boldsymbol{\theta}^*$:
    \begin{equation}
        \begin{aligned}
            \mathcal{L}_{\mu_i} \left(\boldsymbol{\theta}\right)
             & = \mathcal{L}_{\mu_i} \left(\boldsymbol{\theta}^*\right) + (\boldsymbol{\theta} - \boldsymbol{\theta}^*)^\top \nabla_{\boldsymbol{\theta}} \mathcal{L}_{\mu_i} \left(\boldsymbol{\theta}^*\right)  + \frac{1}{2}(\boldsymbol{\theta} - \boldsymbol{\theta}^*)^\top\mathbf{H}_{\mu_i}(\boldsymbol{\theta} - \boldsymbol{\theta}^*) + o(\|\boldsymbol{\theta} - \boldsymbol{\theta}^*\|^2_2)
        \end{aligned}
        \label{taylor:mu_i_*}
    \end{equation}
    \begin{equation}
        \begin{aligned}
            \mathcal{L}_{\mu_j}\left(\boldsymbol{\theta}\right)
             & = \mathcal{L}_{\mu_j}\left(\boldsymbol{\theta}^*\right) + (\boldsymbol{\theta} - \boldsymbol{\theta}^*)^\top \nabla_{\boldsymbol{\theta}} \mathcal{L}_{\mu_j}\left(\boldsymbol{\theta}^*\right) + \frac{1}{2}(\boldsymbol{\theta} - \boldsymbol{\theta}^*)^\top\mathbf{H}_{\mu_j}(\boldsymbol{\theta} - \boldsymbol{\theta}^*) + o(\|\boldsymbol{\theta} - \boldsymbol{\theta}^*\|^2_2)
        \end{aligned}
        \label{taylor:mu_j_*}
    \end{equation}

    Combining the two equations above, we have:
    \begin{equation*}
        \begin{aligned}
                 & {\mathcal{L}}_{\mu_j}\left(\boldsymbol{\theta}\right)-{\mathcal{L}}_{\mu_j}\left(\boldsymbol{\theta}^*_j\right)-\left({\mathcal{L}}_{\mu_i} \left(\boldsymbol{\theta}\right)-{\mathcal{L}}_{\mu_i} \left(\boldsymbol{\theta}^*_i\right)\right)                                                                                                                                             \\
            =    & \mathcal{L}_{\mu_j}\left(\boldsymbol{\theta}^*\right)- \mathcal{L}_{\mu_j}\left(\boldsymbol{\theta}^*_j\right) + \mathcal{L}_{\mu_i} \left(\boldsymbol{\theta}^*\right) - \mathcal{L}_{\mu_i} \left(\boldsymbol{\theta}^*_i\right) + (\boldsymbol{\theta} - \boldsymbol{\theta}^*)^\top \left(\nabla_{\boldsymbol{\theta}} \mathcal{L}_{\mu_j}\left(\boldsymbol{\theta}^*\right) - \nabla_{\boldsymbol{\theta}} \mathcal{L}_{\mu_i} \left(\boldsymbol{\theta}^*\right)\right) \\
                 & +  \frac{1}{2}(\boldsymbol{\theta} - \boldsymbol{\theta}^*)^\top\left(\mathbf{H}_{\mu_j}\left(\boldsymbol{\theta}^*\right) - \mathbf{H}_{\mu_i} \left(\boldsymbol{\theta}^*\right)\right)(\boldsymbol{\theta} - \boldsymbol{\theta}^*) + o(\|\boldsymbol{\theta} - \boldsymbol{\theta}^*\|^2_2)                                                                                                                                               \\
            \leq & \mathcal{L}_{\mu_j}\left(\boldsymbol{\theta}^*\right)- \mathcal{L}_{\mu_j}\left(\boldsymbol{\theta}^*_j\right) + \mathcal{L}_{\mu_i} \left(\boldsymbol{\theta}^*\right) - \mathcal{L}_{\mu_i} \left(\boldsymbol{\theta}^*_i\right) + \|\boldsymbol{\theta} - \boldsymbol{\theta}^*\|_2 \|\nabla_{\boldsymbol{\theta}} \mathcal{L}_{\mu_j}\left(\boldsymbol{\theta}^*\right) - \nabla_{\boldsymbol{\theta}} \mathcal{L}_{\mu_i} \left(\boldsymbol{\theta}^*\right)\|_2         \\
                 & +  \frac{1}{2}\|\boldsymbol{\theta} - \boldsymbol{\theta}^*\|^2_2 \left(\mathbf{H}_{\mu_j}\left(\boldsymbol{\theta}^*\right) - \mathbf{H}_{\mu_i} \left(\boldsymbol{\theta}^*\right)\right) + o(\|\boldsymbol{\theta} - \boldsymbol{\theta}^*\|^2_2)
        \end{aligned}
    \end{equation*}

    Taking the supremum over $\boldsymbol{\theta} \in \mathcal{F}_2$ on both sides, for any $i \neq j$, we have:
    \begin{equation*}
        \begin{aligned}
                 & \sup_{\boldsymbol{\theta} \in \mathcal{F}_2}  \mathcal{L}_{\mu_j}\left(\boldsymbol{\theta}\right)- \mathcal{L}_{\mu_j}\left(\boldsymbol{\theta}^*_j\right) - \left(\mathcal{L}_{\mu_i} \left(\boldsymbol{\theta}\right) - \mathcal{L}_{\mu_i} \left(\boldsymbol{\theta}^*_i\right)\right)                                                                                                                                        \\
            =    & \sup_{\boldsymbol{\theta} \in \mathcal{F}_2}  \mathcal{L}_{\mu_j}\left(\boldsymbol{\theta}^*\right)- \mathcal{L}_{\mu_j}\left(\boldsymbol{\theta}^*_j\right) + \mathcal{L}_{\mu_i} \left(\boldsymbol{\theta}^*\right) - \mathcal{L}_{\mu_i} \left(\boldsymbol{\theta}^*_i\right) + \|\boldsymbol{\theta} - \boldsymbol{\theta}^*\|_2 \|\nabla_{\boldsymbol{\theta}} \mathcal{L}_{\mu_j}\left(\boldsymbol{\theta}^*\right) - \nabla_{\boldsymbol{\theta}} \mathcal{L}_{\mu_i} \left(\boldsymbol{\theta}^*\right)\|_2 \\
                 & \quad \quad + \frac{1}{2}\|\boldsymbol{\theta} - \boldsymbol{\theta}^*\|^2_2 \|\mathbf{H}_{\mu_j}\left(\boldsymbol{\theta}^*\right) - \mathbf{H}_{\mu_i} \left(\boldsymbol{\theta}^*\right)\|_2 + o(\|\boldsymbol{\theta} - \boldsymbol{\theta}^*\|^2_2)                                                                                                                                                                                           \\
            \leq & \mathcal{L}_{\mu_j}\left(\boldsymbol{\theta}^*\right)- \mathcal{L}_{\mu_j}\left(\boldsymbol{\theta}^*_j\right) + \mathcal{L}_{\mu_i} \left(\boldsymbol{\theta}^*\right) - \mathcal{L}_{\mu_i} \left(\boldsymbol{\theta}^*_i\right) + \delta^{\frac{1}{2}} \|\nabla_{\boldsymbol{\theta}} \mathcal{L}_{\mu_j}\left(\boldsymbol{\theta}^*\right) - \nabla_{\boldsymbol{\theta}} \mathcal{L}_{\mu_i} \left(\boldsymbol{\theta}^*\right)\|_2  \\
             & + \frac{1}{2}\delta  \|\mathbf{H}_{\mu_j}\left(\boldsymbol{\theta}^*\right) - \mathbf{H}_{\mu_i} \left(\boldsymbol{\theta}^*\right)\|_2 + o(\delta)
        \end{aligned}
    \end{equation*}
    Finally, maximizing over $i \neq j$ on both sides, the transfer measure is bounded by:
    \begin{equation}
        \begin{aligned}
            \mathrm{T}_\Gamma (\mathcal{S}\|\mathcal{T}) \leq & \frac{1}{2} \max _{i \neq j} \sup_{\boldsymbol{\theta} \in \mathcal{F}_2} \left(\mathcal{L}_{\mu_j}\left(\boldsymbol{\theta}\right) - \mathcal{L}_{\mu_j}\left(\boldsymbol{\theta}^*\right) - \left(\mathcal{L}_{\mu_i} \left(\boldsymbol{\theta}\right) - \mathcal{L}_{\mu_i} \left(\boldsymbol{\theta}^*\right)\right)\right) \\
            \leq                                              & \frac{1}{2} \max_{i \neq j} \mathcal{L}_{\mu_j}\left(\boldsymbol{\theta}^*\right)- \mathcal{L}_{\mu_j}\left(\boldsymbol{\theta}^*_j\right) + \mathcal{L}_{\mu_i} \left(\boldsymbol{\theta}^*\right) - \mathcal{L}_{\mu_i} \left(\boldsymbol{\theta}^*_i\right)\\
            & +  \delta^{\frac{1}{2}} \|\nabla_{\boldsymbol{\theta}} \mathcal{L}_{\mu_j}\left(\boldsymbol{\theta}^*\right) - \nabla_{\boldsymbol{\theta}} \mathcal{L}_{\mu_i} \left(\boldsymbol{\theta}^*\right)\|_2     +  \frac{1}{2}\delta \|\mathbf{H}_{\mu_j}\left(\boldsymbol{\theta}^*\right) - \mathbf{H}_{\mu_i} \left(\boldsymbol{\theta}^*\right)\|_2 + o(\delta) 
        \end{aligned}
    \end{equation}
    Suppose a constant upper bound $g$ on the maximum gradient norm exists, replacing the first-order term with $2 \delta^{\frac{1}{2}} g$, we get:
        \begin{equation}
        \begin{aligned}
            \mathrm{T}_\Gamma (\mathcal{S}\|\mathcal{T}) 
            \leq                                              & \frac{1}{2} \max_{i \neq j} \mathcal{L}_{\mu_j}\left(\boldsymbol{\theta}^*\right)- \mathcal{L}_{\mu_j}\left(\boldsymbol{\theta}^*_j\right) + \mathcal{L}_{\mu_i} \left(\boldsymbol{\theta}^*\right) - \mathcal{L}_{\mu_i} \left(\boldsymbol{\theta}^*_i\right) +  2 \delta^{\frac{1}{2}} g                         \\
                                                              & + \frac{1}{2}\delta \|\mathbf{H}_{\mu_j}\left(\boldsymbol{\theta}^*\right) - \mathbf{H}_{\mu_i} \left(\boldsymbol{\theta}^*\right)\|_2 + o(\delta)\\
                                                               & \leq \delta^{\frac{1}{2}} g + \frac{1}{2} \max_{i \neq j} \mathcal{L}_{\mu_j}\left(\boldsymbol{\theta}^*\right)- \mathcal{L}_{\mu_j}\left(\boldsymbol{\theta}^*_j\right) + \mathcal{L}_{\mu_i} \left(\boldsymbol{\theta}^*\right) - \mathcal{L}_{\mu_i} \left(\boldsymbol{\theta}^*_i\right)                        \\
                                                              & + \frac{1}{2}\delta \|\mathbf{H}_{\mu_j}\left(\boldsymbol{\theta}^*\right) - \mathbf{H}_{\mu_i} \left(\boldsymbol{\theta}^*\right)\|_2 + o(\delta)
                                                              \qedhere
        \end{aligned}
    \end{equation}
\end{proof}

\section{Other Transfer Measures}\label{app:other_tm}
\begin{restatable}[\textbf{upper bounds on symmetric and realizable transfer measures}]{proposition}{propother}
    Given $\mathcal{S} = \left\{\mu_i\right\}_{i=1}^K$ and some $\Gamma \subseteq \mathcal{H}$. Define $\mathcal{L}_{\mu_i}^* : = \inf_{h \in \Gamma} \mathcal{L}_{\mu_i}(h)$ for all $i \in [K]$, $\mathcal{L_T}^* : = \inf_{h \in \Gamma} \mathcal{L_T}(h)$, $\mu^* := \argmin_{\substack{\mu}} \max_{\substack{i \in [K]}}
        \mathrm{T}_{\Gamma} \left(\mu_i\| \mu\right)$, and $\mathcal{L}_\mathcal{S} \left(h\right) := \mathcal{L}_{\mu^*} \left(h\right)$. Under \Cref{assumption:convex_combination}, we have:
    \begin{equation}
        \begin{aligned}
             & \mathrm{T}_{\Gamma}(\mathcal{S}, \mathcal{T}) \leq \frac{1}{2} \max _{i \neq j} \mathrm{T}_{\Gamma}\left(\mu_j \| \mu_i\right)              \\
             & \mathrm{T}^r_{\Gamma}(\mathcal{S}, \mathcal{T}) \leq  \frac{1}{2}\max _{i \neq j} \mathrm{T}_{\Gamma}^{\mathrm{r}}\left(\mu_j, \mu_i\right)
        \end{aligned}
    \end{equation}
\end{restatable}
\begin{proof}
    We first prove an upper bound on symmetric transfer measure $\mathrm{T}_{\Gamma}(\mathcal{S}, \mathcal{T})$.

    From \Cref{def:multi_transfer_measures} and \Cref{eq:one_sided_upperbound}, we have:
    \begin{equation}\label{eq:symmetric_upperbound}
        \begin{aligned}
            \mathrm{T}_{\Gamma}(\mathcal{S}, \mathcal{T}) & := \max \left\{\mathrm{T}_{\Gamma}(\mathcal{S} \| \mathcal{T}), \mathrm{T}_{\Gamma}(\mathcal{T} \| \mathcal{S})\right\}                                                           \\
                                                          & \leq  \max \left\{\frac{1}{2} \max _{i \neq j} \mathrm{T}_{\Gamma}\left(\mu_j \| \mu_i\right), \frac{1}{2}\max _{i \neq j} \mathrm{T}_{\Gamma}\left(\mu_i \| \mu_j\right)\right\} \\
                                                          & =  \frac{1}{2} \max _{i \neq j} \mathrm{T}_{\Gamma}\left(\mu_j \| \mu_i\right)
        \end{aligned}
    \end{equation}

    Now we prove an upper bound on realizable transfer measure $\mathrm{T}_{\Gamma}^{\mathrm{r}}(\mathcal{S}, \mathcal{T})$

    First, define $\mu^* := \argmin_{\substack{\nu}} \max_{\substack{i \in [K]}} \mathrm{T}_{\Gamma} \left(\mu_i, \mu\right)$, and since $\mathcal{T}$ is a convex combination of distribution in $\mathcal{S}$:
    \begin{equation}
        \begin{aligned}
            \mathrm{T}_{\Gamma}^{\mathrm{r}}(\mathcal{S}, \mathcal{T}) & =\sup _{h \in \Gamma}\left|\mathcal{L}_\mathcal{T}(h)-\mathcal{L}_\mathcal{S}(h)\right|                      \\
                                                                       & =\sup _{h \in \Gamma}\left|\sum_{i=1}^K w_i \mathcal{L}_{\mu_i}(h)-\mathcal{L}_{\mu^*}(h)\right|             \\
                                                                       & =\sup _{h \in \Gamma}\left|\sum_{i=1}^K w_i\left[\mathcal{L}_{\mu_i}(h)-\mathcal{L}_{\mu^*}(h)\right]\right| \\
                                                                       & \leq \sum_{i=1}^K w_i \sup _{h \in \Gamma}\left|\mathcal{L}_{\mu_i}(h)-\mathcal{L}_{\mu^*}(h)\right|         \\
                                                                       & =\sum_{i=1}^K w_i \mathrm{T}_{\Gamma}^{\mathrm{r}}\left(\mu_i, \mu^*\right)                                  \\
                                                                       & \leq \max _{i \in[k]} \mathrm{T}_{\Gamma}^{\mathrm{r}}\left(\mu_i, \mu^*\right)
        \end{aligned}
        \label{eq:realizable_upperbound1}
    \end{equation}
    Similar to one-sided transfer measure, let $j_{max}:= \argmax _{j \in [K]}  \mathrm{T}_{\Gamma}^{\mathrm{r}}\left(\mu_j, \mu^*\right)$, and for any fixed , let $\mu_{mid}$ be such that $ \mathrm{T}_{\Gamma}^{\mathrm{r}}\left(\mu_{j_{max}}, \mu_{i}\right) = \mathrm{T}_{\Gamma}^{\mathrm{r}}\left(\mu_{j_{max}}, \mu_{mid}\right) + \mathrm{T}_{\Gamma}^{\mathrm{r}}\left(\mu_{mid}, \mu_{i} \right)$. By the definition $\mu^*$ we have:
    \begin{equation}
        \begin{aligned}
            \frac{1}{2} \max _{i \neq j} \mathrm{T}_{\Gamma}^{\mathrm{r}}\left(\mu_{j} , \mu_i\right)
            \geq & \frac{1}{2} \mathrm{T}_{\Gamma}^{\mathrm{r}}\left(\mu_{j_{max}} , \mu_{i}\right) \quad \forall i \in [K]    \\
            =    & \mathrm{T}_{\Gamma}^{\mathrm{r}}\left(\mu_{mid} ,\mu_{i} \right) \quad \forall i \in [K], \exists \mu_{mid} \\
            \geq & \mathrm{T}_{\Gamma}^{\mathrm{r}}\left(\mu^* , \mu_{i}\right) \quad \forall i \in [K]                        \\
        \end{aligned}
        \label{eq:realizable_upperbound2}
    \end{equation}
    Finally, combining \Cref{eq:realizable_upperbound1} and \Cref{eq:realizable_upperbound2}, we have:
    \begin{equation}
        \mathrm{T}_{\Gamma}^{\mathrm{r}}(\mathcal{S}, \mathcal{T}) \leq \frac{1}{2}\max _{i \neq j} \mathrm{T}_{\Gamma}^{\mathrm{r}}\left(\mu_j, \mu_i\right) \qedhere
        \label{eq:realizable_upperbound}
    \end{equation}
\end{proof}

\section{More Related Work}\label{app:related_works}
\textbf{Domain Generalization.} The goal of domain generalization is to learn a predictor using labeled data from multiple source domains that generalize well to related but unseen target domains~\citep{blanchard_generalizing_2011, muandet_domain_2013}. The standard baseline for DG is Empirical Risk Minimization (ERM)~\citep{vapnik_overview_1999}, which minimizes the average loss across training domains. However, ERM does not generalize well under distribution shifts in the presence of spurious correlation in data~\citep{arjovsky_invariant_2020}. Various approaches have been proposed to address the shortcomings of ERM. Below we discuss some approaches relevant to this work, Invariant Risk Minimization, gradient matching, and hessian matching.

\textbf{Invariant Risk Minimization.} The Invariant Risk Minimization (IRM) principle~\citep{arjovsky_invariant_2020} proposes jointly learning a feature extractor and a classifier such that the optimal classifier remains consistent across different training environments. The IRM objective, by definition, is non-convex and bi-level, so the authors proposed IRMv1, a regularized objective in place of the bi-level one. Later, we make the connection between our proposed loss objective and IRMv1. Followup works~\citep{rosenfeld_risks_2021, ahuja_empirical_2022, wang_provable_2022,wang_invariant-feature_2023,krueger_out--distribution_2021,ahuja_invariance_2022, ahuja_invariant_2020,kamath_does_2021} showed that IRM and its variants do not improve over ERM unless the test domain are similar enough to the training domains.

\textbf{Gradient Matching.} Gradient matching methods seek alignment between domain-level gradients. For instance,  IGA~\citep{koyama_out--distribution_2020} penalizes large Euclidean distances between gradients, Fish~\citep{shi_gradient_2021} increases the gradient inner products, and AND-Mask~\citep{parascandolo_learning_2020} only updates the parameters whose gradients are of the same sign across all environments. Despite their good performance, \citet{hemati_understanding_2023} showed that aligning domain-level gradients does not guarantee small generalization loss to the test domain.

\textbf{Hessian Matching.} Most relevant to our approach, a recent line of DG works align the domain level Hessians w.r.t. the classifier head to promote consistency~\citep{parascandolo_learning_2020} across domains. Due to the complexity of computing the Hessian matrices, prior works find Hessian approximations instead. CORAL~\citet{sun_deep_2016} minimizes the difference in feature covariance matrices between source and target domains, which is approximately Hessian matching. Fishr~\citep{rame_fishr_2022} uses domain-level gradient variance as its hessian approximation. The idea of aligning gradients and Hessian simultaneously was first proposed by \citet{hemati_understanding_2023}, who also discussed what attributes are aligned by gradients and Hessian matching.

\textbf{Domain-Invariant Feature Learning.} Initially proposed by \citet{ben-david_theory_2010}, invariant representation learning seeks various types of invariance across domains. For instance, ~\citet{ganin_domain-adversarial_2016, li_domain_2018,tzeng_adversarial_2017, hoffman_cycada_2017} employ adversarial training, whereas \citet{muandet_domain_2013, long_learning_2015} uses kernel method, \citet{huang_winning_2025} seeks invariant parameters, and \citet{peng_moment_2019, zellinger_central_2019, sun_deep_2016} match the feature moments for domain adaptation. In particular, \citet{sun_deep_2016} introduces CORAL, which matches the covariance between features in the source and target domains and achieves state-of-the-art performance as evaluated by \citet{gulrajani_search_2020} and \citet{hemati_understanding_2023}. Most of the invariant representation learning methods are originally for domain adaptation, where one has access to unlabelled data from the test domain. In the case of multi-domain generalization, these methods can be adopted by finding invariance across training domains. Nevertheless, ~\citet{zhao_learning_2019} shows that matching the features is insufficient for DG.

\section{Connection Between CMA and Existing Methods}\label{app:cma_conection}
By alignment of both gradient and Hessians in closed form, CMA implicitly integrates multiple existing algorithms. Below we build such connections.
\subsection{CMA as Invariant Risk Minimization}
We draw connections between IRM and CMA objectives. Fixing a feature extractor and letting the classifier head be parameterized by $\theta$, the IRMv1 objective in \citet{arjovsky_invariant_2020} is:
\begin{equation}
    \begin{aligned}
        \mathcal{L}_{\text{IRM}} := \mathcal{L}_{\text{ERM}} +  \lambda \frac{1}{K}\sum_{i = 1}^K \| \nabla_{\boldsymbol{\theta}} \mathcal{L}_{\mu_i}\left(\boldsymbol{\theta}\right) \|_2^2
    \end{aligned}
    \tag{IRMv1}
\end{equation}
On the other hand, we can rewrite the gradient variance regularization in \Cref{eq:cma_loss} as
\begin{equation}\label{eq:gv_reg}
    \begin{aligned}
        \frac{1}{K}\sum_{i=1}^K  \| \nabla_{\boldsymbol{\theta}} \mathcal{L}_{\mu_i}\left(\boldsymbol{\theta}\right) - \overline{\nabla_{\boldsymbol{\theta}} \mathcal{L}\left(\boldsymbol{\theta}\right)}\|_2^2 & =  \frac{1}{K}\sum_{i=1}^K  \| \nabla_{\boldsymbol{\theta}} \mathcal{L}_{\mu_i}\left(\boldsymbol{\theta}\right)\|^2_2 - \|\frac{1}{K}\sum_{j=1}^K \nabla_{\boldsymbol{\theta}} \mathcal{L}_{\mu_j}\left(\boldsymbol{\theta}\right)\|_2^2 \\
    \end{aligned}
\end{equation}
The second term on the right-hand side, the norm of the average gradients, is small for a classifier $\boldsymbol{\theta^*}$ well-trained on $\mathcal{L}_{\text{ERM}}$, and the first term resembles the regularization in \Cref{eq:irmv1_loss}. Therefore, penalizing large gradient variance can be seen as enforcing the learned classier $\boldsymbol{\theta}$ to be invariant across domains. Under the same assumptions as in \Cref{thm:moment_alignment_irm}, at the optimal invariant predictor $\boldsymbol{\theta^*}$, the norm of the average of gradients is zero, making the gradient variance term in \Cref{eq:cma_loss} exactly the gradient penalty in \Cref{eq:irmv1_loss}. By setting $\beta = 0$ in \Cref{eq:cma_loss}, we recover \Cref{eq:irmv1_loss}.

\subsection{CMA as Gradient Matching}
While multiple version of gradient matching losses have been proposed~\citep{shi_gradient_2021, koyama_out--distribution_2020, parascandolo_learning_2020}, we focus on the most recent one proposed by \citet{shi_gradient_2021}, defined as:
\begin{equation}\label{eq:gm_loss_2}
    \begin{aligned}
        \mathcal{L}_{\text{GM}} := \mathcal{L}_{\text{ERM}} +\lambda \frac{1}{K} \left( \sum_{i = 1}^K \left\|\nabla_{\boldsymbol{\theta}} \mathcal{L}_{\mu_i}\left(\boldsymbol{\theta}\right)\right\|^2_2 - \left\|\sum_{j = 1}^K \nabla_{\boldsymbol{\theta}} \mathcal{L}_{\mu_{j}}\left(\boldsymbol{\theta}\right)\right\|^2_2 \right)
    \end{aligned}
    \tag{GM}
\end{equation}
Comparing the second term with \Cref{eq:gv_reg}, and ignoring the constant factor $\lambda$, the difference is $\frac{K-1}{K^2} \|\sum_{j = 1}^K \nabla_{\boldsymbol{\theta}} \mathcal{L}_{\mu_{j}}\left(\boldsymbol{\theta}\right)\|^2_2$. When an invariant optimal predictor $\boldsymbol{\theta^*}$ exists, this difference vanishes, and setting $\beta = 0$ in \Cref{eq:cma_loss} recovers \Cref{eq:gm_loss_2}.

\subsection{CMA as Hessian Matching}\label{sec:unif_hm}
We first compare CMA with Fishr~\citep{rame_fishr_2022}, a state-of-the-art DG algorithm based on Hessian matching.
The principle behind Hessian matching is to match the domain-level Hessian matrices by minimizing the objective:
\begin{equation}\label{eq:hm_loss_2} \tag{HM}
    \mathcal{L}_{\text{HM}} := \mathcal{L}_{\text{ERM}} +  \lambda \frac{1}{K}\sum_{i = 1}^K \| \mathbf{H}_{\mu_i} - \overline{\mathbf{H}}\|_F^2
\end{equation}

\citet {rame_fishr_2022} achieves this by approximating the Hessian matrices with their diagonals. In contrast, we proposed to compute the Hessian matrices analytically.
Thus, by setting $\alpha = 0$, \Cref{eq:cma_loss} is the closed-form of the Fishr objective.

Next, we compare CMA with the two objectives proposed in \citet{hemati_understanding_2023}, namely HGP and Hutchinson's method (eq. (18) and eq. (23) in \citet{hemati_understanding_2023}): 
\begin{equation}\label{eq:hgp_loss} \tag{HGP}
    \mathcal{L}_{\text{HGP}} = \mathcal{L}_{\text{ERM}} + \frac{1}{K}\sum_{i=1}^K \alpha \| \nabla_{\boldsymbol{\theta}} \mathcal{L}_{\mu_i} - \overline{\nabla_{\boldsymbol{\theta}} \mathcal{L}}\|_2^2 + \beta \| \mathbf{H}_{\mu_i} \nabla_{\boldsymbol{\theta}} \mathcal{L}_{\mu_i} - \overline{\mathbf{H}\nabla_{\boldsymbol{\theta}} \mathcal{L}}\|_2^2
\end{equation}
where $ \overline{\mathbf{H}\nabla_{\boldsymbol{\theta}} \mathcal{L}} = \frac{1}{K} \sum_{i=1}^K \mathbf{H}_{\mu_i} \nabla_{\boldsymbol{\theta}} \mathcal{L}_{\mu_i}$ is the average Hessian-gradient product.
\begin{equation}\label{eq:hutchinson_loss} \tag{Hutchinson}
    \mathcal{L}_{\text{Hutchinson}} = \mathcal{L}_{\text{ERM}} + \frac{1}{K}\sum_{i=1}^K \alpha \| \nabla_{\boldsymbol{\theta}} \mathcal{L}_{\mu_i} - \overline{\nabla_{\boldsymbol{\theta}} \mathcal{L}}\|_2^2 + \beta \| \mathbf{D}_{\mu_i} - \overline{\mathbf{D}}\|_2^2
\end{equation}
where $\mathbf{D}_{\mu_i}$ is the Hessian diagonal estimated by Hutchinson's method~\citep{bekas_estimator_2007}. Like CMA, HGP, and Hutchinson match the first and second moment across domains. Unlike CMA, HGP approximates the second-order penalties with Hessian-gradient products, while Hutchinson's method estimates them with Hessian diagonals which themselves are estimated by sampling. In other words, \Cref{eq:cma_loss} is the closed form of \Cref{eq:hgp_loss} and \Cref{eq:hutchinson_loss}.

\section{Gradient and Hessian Derivations}\label{app:ce_grad_hess_derivation}
\subsection{Cross-Entropy Loss}
\subsubsection{Gradient}\label{app:grad_softmax}
Given the logistic regression model without a bias term,  parameterized by $ \mathbf{\boldsymbol{\theta}} = \{ \mathbf{w}_1, \ldots, \mathbf{w}_C \}$, where $\mathbf{w}_c \in \mathbb{R}^d$ for all $c \in [C]$, and the prediction $p_c = \frac{e^{\mathbf{w}_c^\top \mathbf{x}}}{\sum_{j=1}^C e^{\mathbf{w}_j^\top \mathbf{x}}} $, the cross-entropy loss for a single example $ \left(\mathbf{x}, y\right) $ is defined as:

\[
    \ell\left( \mathbf{\boldsymbol{\theta}} \right) = - \sum_{c=1}^C y_c \log\left( p_c \right)
\]

To find the gradient of the loss w.r.t. $ \mathbf{w}_k $, we compute:
\begin{equation*}
    \begin{aligned}
        \nabla_{\mathbf{w}_k} \ell (\boldsymbol{\theta}) & = - \sum_{c=1}^C y_c \nabla_{\mathbf{w}_k} \log\left( p_c \right)                                                         \\
                                   & = - \sum_{c \neq k}^C y_c \nabla_{\mathbf{w}_k} \log\left( p_c \right) - y_k \nabla_{\mathbf{w}_k} \log\left( p_k \right) \\
                                   & = \sum_{c \neq k}^C y_c p_k \mathbf{x} - y_k \mathbf{x} \left(1 - p_k\right)                                              \\
                                   & = (1-y_k) p_k \mathbf{x} - y_k \mathbf{x} \left(1 - p_k\right)                                                            \\
                                   & = \left( p_k - y_k \right) \mathbf{x}
    \end{aligned}
\end{equation*}
From the second to the third equality, we use the facts that

\begin{equation*}
    \begin{aligned}
        \nabla_{\mathbf{w}_k}  p_c & = \begin{cases}
                                           p_k \left(1-p_k\right) \mathbf{x} , & \text{if } c = k    \\
                                           - p_c p_k \mathbf{x},               & \text{if } c \neq k
                                       \end{cases}
    \end{aligned}
\end{equation*}
\begin{equation*}
    \begin{aligned}
        \nabla_{\mathbf{w}_k} \log\left( p_c \right) & = \begin{cases}
                                                             \left( 1 - p_k \right) \mathbf{x}, & \text{if } c = k    \\
                                                             -p_k \mathbf{x} ,                  & \text{if } c \neq k
                                                         \end{cases}
    \end{aligned}
\end{equation*}

\subsubsection{Hessian}\label{app:hess_softmax}
To find the Hessian matrix, we compute the second-order partial derivatives. We consider two cases:

Case 1: $ k = c $:
\begin{equation*}
    \begin{aligned}
        \nabla_{\mathbf{w}_k} \nabla_{\mathbf{w}_k} \ell (\boldsymbol{\theta}) & = \nabla_{\mathbf{w}_k} \left( \left( p_k - y_k \right) \mathbf{x} \right) \\
                                                         & = \nabla_{\mathbf{w}_k} p_k \mathbf{x}                                     \\
                                                         & = p_k \left( 1 - p_k \right) \mathbf{x} \mathbf{x}^\top
    \end{aligned}
\end{equation*}

Case 2: $ k \neq c $:
\begin{equation*}
    \begin{aligned}
        \nabla_{\mathbf{w}_k} \nabla_{\mathbf{w}_c} \ell (\boldsymbol{\theta}) & = \nabla_{\mathbf{w}_k} \left( \left( p_c - y_c \right) \mathbf{x} \right) \\
                                                         & = \nabla_{\mathbf{w}_k} p_c \mathbf{x}                                     \\
                                                         & = -p_c p_k \mathbf{x} \mathbf{x}^\top
    \end{aligned}
\end{equation*}

Combining these results, we write the Hessian matrix as:
\begin{equation*}
    \mathbf{H} = \left( \mathrm{diag}(\mathbf{p}) - \mathbf{p} \mathbf{p}^\top \right) \otimes \left( \mathbf{x} \mathbf{x}^\top \right)
\end{equation*}

Where:
\begin{itemize}
    \item $ \mathrm{diag}(\mathbf{p}) \in \mathbb{R}^{C \times C}$ is the diagonal matrix with elements of $ \mathbf{p}$, $p_1,\dots, p_C$, on the diagonal.
    \item $ \mathbf{x} \mathbf{x}^\top \in \mathbb{R}^{d \times d} $.
    \item $ \otimes $ denotes the Kronecker product.
\end{itemize}

\subsubsection{Higher Order Derivatives of Logistic Regression Classifier}\label{app:high_nabla}

We show by induction that the $n^{\text{th}}$ order derivative of the cross-entropy loss w.r.t. the weight vector \(\mathbf{w}\) of a binary-logistic regression classifier is:

\begin{equation}\label{eq:n_derivative}
\mathtt{\nabla^n_{\mathbf{w}}\ell (\boldsymbol{\theta})} = Q_n(p)\mathtt{\mathbf{x}^{\otimes n}}
\end{equation}
where $Q_n(p)$ is some scalar-valued polynomial function of $p$.
\begin{proof} By induction.

\textbf{Base Case (\( n = 1 \))}:

For \( n = 1 \), the gradient of the cross-entropy loss \(\ell\) w.r.t. \(\mathbf{w}\) is:

\[
\nabla^1_{\mathbf{w}} \ell (\boldsymbol{\theta})= (p - y) \mathbf{x}
\]

This matches the form $Q_1(p) \mathbf{x}^{\otimes 1}$ for $Q_1(p) = p - y$. We have that the base case holds.

\textbf{Inductive Step}: Assume \Cref{eq:n_derivative} holds for some $n$

\[\mathtt{\nabla^n_{\mathbf{w}}\ell (\boldsymbol{\theta})} = Q_n(p)\mathtt{\mathbf{x}^{\otimes n}}
\]
we need to show that it also holds for \((n+1)\):

\[
\mathtt{\nabla^{n+1}_{\mathbf{w}} \ell (\boldsymbol{\theta}) }= Q_{n}\left(p\right) \mathtt{\mathbf{x}^{\otimes (n+1)}}
\]

By the product rule:
\begin{equation*}
    \begin{aligned}
    \mathtt{\nabla^{n+1}_{\mathbf{w}} \ell (\boldsymbol{\theta}) }= & \nabla_{\mathbf{w}} Q_{n}\left(p\right) \mathtt{\mathbf{x}^{\otimes n}}\\
    = &\left[\nabla_{\mathbf{w}} Q_{n}\left(p\right) \right]\mathtt{\mathbf{x}^{\otimes n}} 
    \end{aligned}
\end{equation*}
And by chain rule:
\begin{equation*}
    \begin{aligned}
        \nabla_{\mathbf{w}} Q_{n}\left(p\right) 
        = & \left[\nabla_{p} Q_{n}\left(p\right) \right] \nabla_{\mathbf{w}} p
    \end{aligned}
\end{equation*}

The first gradient is the derivative of a polynomial function of $p$, which is again a polynomial function of $p$. The second term, as we have seen in \Cref{app:hess_softmax}, is $p(1-p) \mathbf{x}$. Now putting everything together, we have
\begin{equation*}
    \begin{aligned}
    \mathtt{\nabla^{n+1}_{\mathbf{w}} \ell (\boldsymbol{\theta})}
    = &\left[\nabla_{\mathbf{w}} Q_{n}\left(p\right) \right]\mathtt{\mathbf{x}^{\otimes n}} \\
    = & \left[\nabla_{\mathbf{w}} Q_{n}\left(p\right) \right] p (1-p) \mathbf{x} \mathtt{\mathbf{x}^{\otimes n}}\\
    = & Q_{n+1}\left(p\right) \mathtt{\mathbf{x}^{\otimes n + 1}}
    \end{aligned}
\end{equation*}
which completes the induction.
\qedhere
\end{proof}

\subsubsection{Memory-Efficient Hessian Frobenius Norm}\label{app:eff_hessian}

Note that to obtain the Frobenius of the hessian, we do not need to compute the Kroncker product explicitly:
\begin{equation*}
\begin{aligned}
    \|\mathbf{H}\|^2_F = &\operatorname{tr}\left( \mathrm{diag}(\mathbf{p}) - \mathbf{p} \mathbf{p}^\top \right) \operatorname{tr}\left( \mathbf{x} \mathbf{x}^\top \right)
\end{aligned}
\end{equation*}

To compute the Hessian regularization \begin{equation*}
    \frac{1}{K} \sum_{i = 1}^K \|\mathbf{H}_{\mu_i} \left(\boldsymbol{\theta}\right) - \overline{\mathbf{H}\left(\boldsymbol{\theta}\right)}\|^2_F
\end{equation*}
without saving the $dC \times dC$ Kroncker product, we first expand the Frobenius norm:
\begin{equation*}
    \begin{aligned}
        \|\mathbf{H}_{\mu_i}(\boldsymbol{\theta}) - \overline{\mathbf{H}(\boldsymbol{\theta})}\|^2_F & = \|\mathbf{H}_{\mu_i}(\boldsymbol{\theta})\|^2_F + \left\|\frac{1}{K} \sum_{j=1}^K \mathbf{H}_{\mu_j}(\boldsymbol{\theta})\right\|^2_F - 2 \left\langle \mathbf{H}_{\mu_i}(\boldsymbol{\theta}), \frac{1}{K} \sum_{j=1}^K \mathbf{H}_{\mu_j}(\boldsymbol{\theta}) \right\rangle_F                 \\
                                                                           & = \|\mathbf{H}_{\mu_i}(\boldsymbol{\theta})\|^2_F + \frac{1}{K^2} \left\|\sum_{j=1}^K \mathbf{H}_{\mu_j}(\boldsymbol{\theta})\right\|^2_F - \frac{2}{K^2} \sum_{j=1}^K \left\langle \mathbf{H}_{\mu_i}(\boldsymbol{\theta}), \mathbf{H}_{\mu_j}(\boldsymbol{\theta}) \right\rangle_F               \\
                                                                           & = \|\mathbf{H}_{\mu_i}(\boldsymbol{\theta})\|^2_F + \frac{1}{K^2} \sum_{ j,l=1}^K \langle\mathbf{H}_{\mu_j}(\boldsymbol{\theta}), \mathbf{H}_{\mu_l}(\boldsymbol{\theta})\rangle_F - \frac{2}{K^2} \sum_{j=1}^K \langle\mathbf{H}_{\mu_i}(\boldsymbol{\theta}) \mathbf{H}_{\mu_j}(\boldsymbol{\theta})\rangle_F
    \end{aligned}
\end{equation*}

We need $\left\langle \mathbf{H}_{\mu_i}(\boldsymbol{\theta}),\mathbf{H}_{\mu_j}(\boldsymbol{\theta}) \right\rangle_{F}$ for all $i, j \in [K]$. For the ease of notation, we denote the two environmental Hessians as $\mathbf{H}^{e_1}, \mathbf{H}^{e_2}$, $\mathcal{E}_e$ as the indices of points in environment $e$, and $\mathbf{H}_i$ as the Hessian of the sample $i$.
\begin{equation*}
    \begin{aligned}
        \left\langle H^{e_1}, H^{e_2} \right\rangle_F & = \frac{1}{|\mathcal{E}_{e_1}||\mathcal{E}_{e_2}|} \sum_{i \in \mathcal{E}_{e_1}} \sum_{j \in \mathcal{E}_{e_2}} \left\langle \mathbf{H}_i, \mathbf{H}_j \right\rangle_F                                                                                                                                                                                                                                                               \\
                                                      & = \frac{1}{|\mathcal{E}_{e_1}||\mathcal{E}_{e_2}|} \sum_{i \in \mathcal{E}_{e_1}} \sum_{j \in \mathcal{E}_{e_2}} \operatorname{tr}(\mathbf{H}_i \mathbf{H}_j)                                                                                                                                                                                                                                                                          \\
                                                      & = \frac{1}{|\mathcal{E}_{e_1}||\mathcal{E}_{e_2}|} \sum_{i \in \mathcal{E}_{e_1}} \sum_{j \in \mathcal{E}_{e_2}} \operatorname{tr} \left( (\operatorname{diag}(\mathbf{p}^{(i)}) - \mathbf{p}^{(i)} {\mathbf{p}^{(i)}}^\top) \otimes \mathbf{x}^{(i)} {\mathbf{x}^{(i)}}^\top ((\operatorname{diag}(\mathbf{p}^{(j)}) - \mathbf{p}^{(j)} {\mathbf{p}^{(j)}}^\top) \otimes \mathbf{x}^{(j)} {\mathbf{x}^{(j)}}^\top) \right)            \\
                                                      & = \frac{1}{|\mathcal{E}_{e_1}||\mathcal{E}_{e_2}|} \sum_{i \in \mathcal{E}_{e_1}} \sum_{j \in \mathcal{E}_{e_2}} \operatorname{tr} \left( \operatorname{diag}(\mathbf{p}^{(i)} - \mathbf{p}^{(i)} {\mathbf{p}^{(i)}}^\top) \operatorname{diag}(\mathbf{p}^{(j)} - \mathbf{p}^{(j)} {\mathbf{p}^{(j)}}^\top) \right) \operatorname{tr}\left(\mathbf{x}^{(i)} {\mathbf{x}^{(i)}}^\top) (\mathbf{x}^{(j)} {\mathbf{x}^{(j)}}^\top)\right)
    \end{aligned}
\end{equation*}
The last expression only involves matrices of dimensions $C \times C$ and $d \times d$.

However, this memory-efficient method requires computing the trace for all pairs of Hessians, $\mathbf{H}_i$ and $\mathbf{H}_j$, where $\left(i, j\right) \in \left(\mathcal{E}_{e_1}, \mathcal{E}_{e_2}\right)$ for each combination of environments $e_1, e_2 \in [K]$.

\subsection{Mean-Squared Error Loss}
We derive the gradient and Hessian of the mean-squared error (MSE) loss.
Given a linear regression model parameterized by $\mathbf{w} \in \mathbb{R}^d$, where the prediction is $\hat{y} = \mathbf{w}^\top \mathbf{x}$, the mean-squared error loss for a single example $(\mathbf{x}, y)$ is defined as:
\[
    \ell(\mathbf{w}) = \frac{1}{2} (\hat{y} - y)^2 = \frac{1}{2} (\mathbf{w}^\top \mathbf{x} - y)^2
\]

\subsubsection{Gradient}
To find the gradient of the loss w.r.t. $\mathbf{w}$, we compute:
\begin{equation*}
    \begin{aligned}
        \nabla_{\mathbf{w}} \ell (\boldsymbol{\theta}) & = \nabla_{\mathbf{w}} \frac{1}{2} (\mathbf{w}^\top \mathbf{x} - y)^2                    \\
                                 & = (\mathbf{w}^\top \mathbf{x} - y) \nabla_{\mathbf{w}} (\mathbf{w}^\top \mathbf{x} - y) \\
                                 & = (\mathbf{w}^\top \mathbf{x} - y) \mathbf{x}                                           \\
                                 & = (\hat{y} - y) \mathbf{x}
    \end{aligned}
\end{equation*}
\subsubsection{Hessian}
To find the Hessian matrix, we compute the second-order partial derivatives:
\begin{equation*}
    \begin{aligned}
        \mathbf{H}\left(\mathbf{x}\right) = \nabla_{\mathbf{w}} \nabla_{\mathbf{w}} \ell (\boldsymbol{\theta}) & = \nabla_{\mathbf{w}} (\hat{y} - y) \mathbf{x}                    \\
                                                     & = \nabla_{\mathbf{w}} (\mathbf{w}^\top \mathbf{x} - y) \mathbf{x} \\
                                                     & = \mathbf{x} \mathbf{x}^\top
    \end{aligned}
\end{equation*}
Note that the second-order derivative of MSE loss is a constant matrix w.r.t. $\mathbf{w}$, so higher-order derivatives are tensors with all zeros.

\section{Experimental Details}\label{app:exp}
\subsection{Linear Probing}\label{app:irm_exp}
For the linear probing experiments in \Cref{sec:experiment}, we conduct a grid search for both $\alpha$ and $\beta$ in \Cref{eq:cma_loss} over the set \{1, 10, 100, 1000, 2000, 5000, 10000\}. We also implement penalty annealing, wherein the gradient and Hessian penalties are initially set to zero and activated only after a predetermined number of updates. This approach ensures that the classifier, to which further regularization is subsequently applied, already achieves a small ERM loss. For Fishr, we perform a grid search over the suggested hyperparameter ranges by \citet{rame_fishr_2022}. The grid search is first conducted using a single random seed. From this, the top five performing sets of hyperparameters are chosen. These sets are then evaluated using four additional random seeds. Lastly, we report the test performance of the hyperparameter set that demonstrates the highest worst-group validation accuracy over five runs. Here we summarize the best hyperparameters found for Fishr and CMA.
\begin{table}[h!]
    \centering
    \caption{Best hyperparameters for Fishr and CMA on each dataset.}
    \adjustbox{max width=\textwidth}{%
        \begin{tabular}{llccc}
            \toprule
            Algorithm & Parameter                                 & \textbf{Waterbirds} & \textbf{CelebA} & \textbf{MultiNLI} \\
            \midrule
            Fishr     & regularization strength $\lambda$         & $100$              & $10$            & $ 10000$          \\
                      & ema $\gamma$                              & $0.945$             & $0.9225$        & $0.99$            \\
                      & annealing iterations                      & $2800$              & $12000$         & $600$             \\
            \midrule
            CORAL     & regularization strength $\gamma$         & $0.256$              & $0.16$            & $0.45$          \\
            \midrule
            CMA       & gradient regularization strength $\alpha$ & $10$                & $5000$          & $5000$            \\
                      & hessian regularization strength $\beta$   & $1000$              & $100$           & $ 1$              \\
                      & annealing iterations                      & $2100$              & $ 4000$         & $0$               \\
            \bottomrule
        \end{tabular}}
\end{table}
We train on Waterbirds for 300 epochs, CelebA for 50 epochs, and MultiNLI for 3 epochs. Since ISR projected features have small dimensions (we follow the implementation in \citet{wang_provable_2022} and choose 100), the experiment is computationally efficient to run, taking five days on four RTX 6000 GPUs.
\subsubsection{Datasets}\label{app:exp_irm_data}
\textit{Waterbirds}~\cite{sagawa_distributionally_2020}: This is an image dataset, where each image is a combination of a bird image from the CUB~\citep{wah_caltech-ucsd_2011} and a background image from the Place dataset~\citep{zhou_places_2018}. Each combined image is labelled with class $y \in \mathcal{Y} = \{\textit{waterbird}, \textit{landbird}\}$ and environment $e \in \mathcal{E} = \{\textit{water\_background}, \textit{land\_background}\}$. Each $(y, e)$ pair forms a group, for a total of 4 groups $\mathcal{G} = \mathcal{Y} \times \mathcal{E}$. There are 4795 training samples, and the smallest group has 56.

\textit{CelebA} \cite{liu_deep_2015}: This is an image dataset composed of celebrity faces. Following \citet{sagawa_distributionally_2020} and \citet{wang_provable_2022, wang_invariant-feature_2023}, we consider a hair color classification task ($\mathcal{Y} = \{\textit{blond}, \textit{dark}\}$) with gender as spurious feature ($\mathcal{E} = \{\textit{male}, \textit{female}\}$). The four groups are formed by $\mathcal{G} = \mathcal{Y} \times \mathcal{E}$. There are 162k training samples, and the smallest group, males with blond hair, has 1387 samples.

\textit{MultiNLI}~\citep{williams_broad-coverage_2018}: This is a text dataset for natural language inference. Each sample is composed of one hypothesis and one premise, and the task is to determine whether the given premise entails, is neutral with, or contradicted by the hypothesis ($\mathcal{Y} = \{\textit{contradiction}, \textit{neutral}, \textit{entailment}\}$). The spurious attribute is the presence of negation words, for example, ``no'', ``nobody'', ``never'', and ``nothing'' ($\mathcal{E} = \{\textit{no\_negation}, \textit{negation}\}$). The presence of negation words spuriously correlated with $y = contradiction$~\citep{gururangan_annotation_2018}. There are six groups formed by $\mathcal{G} = \mathcal{Y} \times \mathcal{E}$, for a total of 206175 samples in the training set. The smallest group, entailment with negations, contains 1521 examples.

\subsection{Fine-Tuning}\label{app:non_irm_exp}
For the fine-tuning experiments in \Cref{sec:experiment}, we employ two model selection strategies from DomainBed~\citep{gulrajani_search_2020}: test-domain model selection and training-domain model selection. In test-domain model selection, we select the best hyperparameters based on a validation set that follows the same distribution as the test data. On the other hand, for training-domain model selection, the best hyperparameters are chosen based on performance across holdout sets from the training domains. Contrary to the original DomainBed setup, which randomizes batch sizes, we standardize the batch size to 64 for ColoredMNIST and RotatedMNIST, and to 32 for real image datasets. For each algorithm, we randomly search for 5 sets of hyperparameters and 3 runs each. The experiments take around 10 days on 4 RTX 6000 GPUs.

Despite the original DomainBed codebase recommending a search over 20 sets of hyperparameters per algorithm, per dataset, and per test domain, we restricted our search to only 5 sets due to time and resource constraints. Even with this limitation, our approach required running 1260 experiments. While this reduced number of searches means the algorithms might not have achieved their full potential, this limitation applies equally to all algorithms, ensuring a fair comparison. As our experiments are intended as proof-of-concept rather than comprehensive evaluations, we argue that the results in this section are sufficient to validate the effectiveness of our algorithm.

In the main text, we follow \citet{rame_fishr_2022} and report the test-domain validation performance. In practice, test-domain model selection is more realistic compared to training-domain model selection, as practitioners are unlikely to deploy a model without validating it with at least some small-scale data from the target domain. Additionally, as discussed in \citet{rame_fishr_2022} and \citet{teney_evading_2022}, by the definition of distribution shift, one cannot expect a model selected on a validation set sampled from the same distribution as the training set to generalize to an unseen test distribution. For completeness, we also present the training-domain validation performance in \Cref{app:per_data_result}.

\subsubsection{Datasets}\label{app:non_irm_exp_data}
\textit{Colored MNIST}~\citep{arjovsky_invariant_2020}: This is an image dataset derived from the MNIST handwritten digit classification dataset~\citep{lecun_mnist_2010}. The task is to identify whether a digit is in 0-4 or 5-9 ($\mathcal{Y} = \{0-4, 5-9\}$). The digits are colored red or blue. The environments contain colored digits correlated differently ($\mathcal{E} = \{+90\%, +80\%, -90\%\}$) with the target label. In the first environment, the green color has a 90\% correlation with class 5-9; similar correlations apply in the other two environments. Additionally, there is a 25\% chance of label flipping. The dataset contains 70,000 examples of dimensions (2, 28, 28) categorized into 2 classes.

\textit{Rotated MNIST}~\citep{ghifary_domain_2015}: This is another variant of MNIST, where each environment $e \in \mathcal{E} = \{0, 15, 30, 45, 60, 75\}$ is composed of digits rotated by $e$ degrees. The dataset contains  70,000 examples of dimensions (1, 28, 28) and 10 classes.

\textit{PACS}~\citep{li_deeper_2017}: This is a 7-class classification dataset, where each image is either photo, art painting, cartoon, or sketch ($\mathcal{E} = \{\textit{photo}, \textit{art\_painting}, \textit{cartoon}, \textit{sketch}\}$). There are 9,991 samples, each with dimensions (3, 224, 224).

\textit{VLCS}~\citep{fang_unbiased_2013}: This is a 5 class images dataset with images from environments \ $\mathcal{E} = \{\textit{Caltech101}, \textit{LabelMe}, \textit{SUN09}, \textit{VOC2007}\}$. There are 10,729 samples of dimension (3, 224, 224).

\textit{Terra Incognita}~\citep{beery_recognition_2018}: This is a dataset of photographs taken from various locations, each corresponds to one environment ($\mathcal{E} = \{\textit{L100}, \textit{L38}, \textit{L43}, \textit{L46}\}$). The DomainBed benchmark includes a subset of Terra Incognita, comprising 24,788 samples with dimensions (3, 224, 224) across 10 classes.

\subsection{DomainBed results}\label{app:per_data_result}
\subsubsection{Model selection: training-domain validation set}

\textbf{ColoredMNIST}

\begin{center}
\adjustbox{max width=\textwidth}{%
\begin{tabular}{lcccc}
\toprule
\textbf{Algorithm}   & \textbf{+90\%}       & \textbf{+80\%}       & \textbf{-90\%}       & \textbf{Avg}         \\
\midrule
ERM                  & 72.2 $\pm$ 0.2       & 72.9 $\pm$ 0.2       & 10.1 $\pm$ 0.1       & 51.7                 \\
CORAL                & 71.7 $\pm$ 0.4       & 73.2 $\pm$ 0.1       & 10.2 $\pm$ 0.1       & 51.7                 \\
Fishr                & 72.6 $\pm$ 0.3       & 73.3 $\pm$ 0.1       & 10.6 $\pm$ 0.2       & 52.2                 \\
CMA     & 71.4 $\pm$ 0.3       & 72.8 $\pm$ 0.1       & 10.0 $\pm$ 0.2       & 51.4                 \\
\bottomrule
\end{tabular}}
\end{center}

\textbf{RotatedMNIST}

\begin{center}
\adjustbox{max width=\textwidth}{%
\begin{tabular}{lccccccc}
\toprule
\textbf{Algorithm}   & \textbf{0}           & \textbf{15}          & \textbf{30}          & \textbf{45}          & \textbf{60}          & \textbf{75}          & \textbf{Avg}         \\
\midrule
ERM                  & 95.3 $\pm$ 0.2       & 98.6 $\pm$ 0.1       & 99.1 $\pm$ 0.1       & 98.9 $\pm$ 0.0       & 98.9 $\pm$ 0.0       & 96.1 $\pm$ 0.2       & 97.8                 \\
CORAL                & 95.7 $\pm$ 0.2       & 98.7 $\pm$ 0.1       & 99.0 $\pm$ 0.0       & 99.0 $\pm$ 0.0       & 99.0 $\pm$ 0.0       & 96.5 $\pm$ 0.0       & 98.0                 \\
Fishr                & 95.6 $\pm$ 0.3       & 98.5 $\pm$ 0.1       & 99.1 $\pm$ 0.1       & 99.0 $\pm$ 0.1       & 99.0 $\pm$ 0.1       & 96.4 $\pm$ 0.0       & 97.9                 \\
CMA     & 95.2 $\pm$ 0.2       & 98.4 $\pm$ 0.2       & 98.9 $\pm$ 0.0       & 98.9 $\pm$ 0.0       & 98.9 $\pm$ 0.1       & 96.5 $\pm$ 0.2       & 97.8                 \\
\bottomrule
\end{tabular}}
\end{center}

\textbf{VLCS}

\begin{center}
\adjustbox{max width=\textwidth}{%
\begin{tabular}{lccccc}
\toprule
\textbf{Algorithm}   & \textbf{C}           & \textbf{L}           & \textbf{S}           & \textbf{V}           & \textbf{Avg}         \\
\midrule
ERM                  & 97.1 $\pm$ 0.1       & 62.3 $\pm$ 0.3       & 71.9 $\pm$ 0.7       & 77.2 $\pm$ 0.4       & 77.2                 \\
CORAL                & 96.3 $\pm$ 0.1       & 64.5 $\pm$ 0.4       & 72.4 $\pm$ 0.3       & 72.4 $\pm$ 1.7       & 76.4                 \\
Fishr                & 96.4 $\pm$ 0.6       & 63.3 $\pm$ 0.9       & 74.8 $\pm$ 0.6       & 76.2 $\pm$ 0.4       & 77.7                 \\
CMA     & 96.1 $\pm$ 0.6       & 63.2 $\pm$ 0.4       & 73.5 $\pm$ 0.4       & 78.9 $\pm$ 0.3       & 77.9                 \\
\bottomrule
\end{tabular}}
\end{center}

\textbf{PACS}

\begin{center}
\adjustbox{max width=\textwidth}{%
\begin{tabular}{lccccc}
\toprule
\textbf{Algorithm}   & \textbf{A}           & \textbf{C}           & \textbf{P}           & \textbf{S}           & \textbf{Avg}         \\
\midrule
ERM                  & 80.2 $\pm$ 0.6       & 75.4 $\pm$ 0.2       & 95.9 $\pm$ 0.8       & 66.6 $\pm$ 0.3       & 79.5                 \\
CORAL                & 81.6 $\pm$ 0.6       & 74.9 $\pm$ 0.8       & 95.4 $\pm$ 0.6       & 64.9 $\pm$ 0.6       & 79.2                 \\
Fishr                & 83.1 $\pm$ 1.0       & 74.8 $\pm$ 0.5       & 97.2 $\pm$ 0.2       & 68.7 $\pm$ 0.8       & 81.0                 \\
CMA     & 83.3 $\pm$ 0.3       & 76.4 $\pm$ 0.2       & 96.1 $\pm$ 0.1       & 66.3 $\pm$ 0.7       & 80.5                 \\
\bottomrule
\end{tabular}}
\end{center}

\textbf{TerraIncognita}

\begin{center}
\adjustbox{max width=\textwidth}{%
\begin{tabular}{lccccc}
\toprule
\textbf{Algorithm}   & \textbf{L100}        & \textbf{L38}         & \textbf{L43}         & \textbf{L46}         & \textbf{Avg}         \\
\midrule
ERM                  & 48.2 $\pm$ 2.1       & 17.8 $\pm$ 2.3       & 37.8 $\pm$ 1.0       & 34.2 $\pm$ 0.5       & 34.5                 \\
CORAL                & 39.1 $\pm$ 2.1       & 12.4 $\pm$ 2.1       & 36.0 $\pm$ 1.4       & 30.6 $\pm$ 0.9       & 29.5                 \\
Fishr                & 47.2 $\pm$ 2.1       & 16.5 $\pm$ 1.6       & 39.9 $\pm$ 1.9       & 33.2 $\pm$ 0.7       & 34.2                 \\
CMA     & 45.8 $\pm$ 3.3       & 19.0 $\pm$ 1.2       & 37.7 $\pm$ 0.3       & 33.4 $\pm$ 1.0       & 34.0                 \\
\bottomrule
\end{tabular}}
\end{center}

\textbf{Averages}

\begin{center}
\adjustbox{max width=\textwidth}{%
\begin{tabular}{lcccccc}
\toprule
\textbf{Algorithm}        & \textbf{ColoredMNIST}     & \textbf{RotatedMNIST}     & \textbf{VLCS}             & \textbf{PACS}             & \textbf{TerraIncognita}   & \textbf{Avg}              \\
\midrule
ERM                       & 51.7 $\pm$ 0.1            & 97.8 $\pm$ 0.1            & 77.2 $\pm$ 0.2            & 79.5 $\pm$ 0.3            & 34.5 $\pm$ 0.4            & 68.1                      \\
CORAL                     & 51.7 $\pm$ 0.1            & 98.0 $\pm$ 0.0            & 76.4 $\pm$ 0.5            & 79.2 $\pm$ 0.1            & 29.5 $\pm$ 1.1            & 67.0                      \\
Fishr                     & 52.2 $\pm$ 0.1            & 97.9 $\pm$ 0.1            & 77.7 $\pm$ 0.4            & 81.0 $\pm$ 0.3            & 34.2 $\pm$ 0.9            & 68.6                      \\
CMA          & 51.4 $\pm$ 0.0            & 97.8 $\pm$ 0.0            & 77.9 $\pm$ 0.1            & 80.5 $\pm$ 0.2            & 34.0 $\pm$ 0.7            & 68.3                      \\
\bottomrule
\end{tabular}}
\end{center}

\subsubsection{Model selection: test-domain validation set (oracle)}\label{app:perdata_oracle}

\textbf{ColoredMNIST}

\begin{center}
\adjustbox{max width=\textwidth}{%
\begin{tabular}{lcccc}
\toprule
\textbf{Algorithm}   & \textbf{+90\%}       & \textbf{+80\%}       & \textbf{-90\%}       & \textbf{Avg}         \\
\midrule
ERM                  & 68.1 $\pm$ 1.1       & 70.5 $\pm$ 0.7       & 25.0 $\pm$ 1.9       & 54.5                 \\
CORAL                & 68.2 $\pm$ 0.9       & 72.0 $\pm$ 0.8       & 26.9 $\pm$ 0.1       & 55.7                 \\
Fishr                & 73.9 $\pm$ 0.3       & 73.5 $\pm$ 0.2       & 38.5 $\pm$ 5.2       & 62.0                 \\
CMA     & 70.9 $\pm$ 0.6       & 72.2 $\pm$ 0.2       & 44.3 $\pm$ 2.9       & 62.5                 \\
\bottomrule
\end{tabular}}
\end{center}

\textbf{RotatedMNIST}

\begin{center}
\adjustbox{max width=\textwidth}{%
\begin{tabular}{lccccccc}
\toprule
\textbf{Algorithm}   & \textbf{0}           & \textbf{15}          & \textbf{30}          & \textbf{45}          & \textbf{60}          & \textbf{75}          & \textbf{Avg}         \\
\midrule
ERM                  & 95.2 $\pm$ 0.3       & 98.5 $\pm$ 0.1       & 98.9 $\pm$ 0.1       & 98.9 $\pm$ 0.1       & 99.0 $\pm$ 0.1       & 96.2 $\pm$ 0.2       & 97.8                 \\
CORAL                & 95.8 $\pm$ 0.1       & 98.7 $\pm$ 0.1       & 98.9 $\pm$ 0.0       & 99.2 $\pm$ 0.1       & 99.1 $\pm$ 0.0       & 96.5 $\pm$ 0.1       & 98.0                 \\
Fishr                & 95.7 $\pm$ 0.2       & 98.7 $\pm$ 0.0       & 99.0 $\pm$ 0.1       & 99.1 $\pm$ 0.1       & 98.8 $\pm$ 0.2       & 96.4 $\pm$ 0.0       & 97.9                 \\
CMA     & 95.7 $\pm$ 0.2       & 98.8 $\pm$ 0.1       & 98.9 $\pm$ 0.1       & 98.9 $\pm$ 0.0       & 98.9 $\pm$ 0.1       & 95.9 $\pm$ 0.6       & 97.9                 \\
\bottomrule
\end{tabular}}
\end{center}

\textbf{VLCS}

\begin{center}
\adjustbox{max width=\textwidth}{%
\begin{tabular}{lccccc}
\toprule
\textbf{Algorithm}   & \textbf{C}           & \textbf{L}           & \textbf{S}           & \textbf{V}           & \textbf{Avg}         \\
\midrule
ERM                  & 96.4 $\pm$ 0.1       & 62.3 $\pm$ 1.0       & 72.1 $\pm$ 0.6       & 76.7 $\pm$ 0.3       & 76.9                 \\
CORAL                & 95.8 $\pm$ 0.3       & 63.1 $\pm$ 0.3       & 71.2 $\pm$ 0.3       & 73.5 $\pm$ 0.2       & 75.9                 \\
Fishr                & 96.0 $\pm$ 0.8       & 64.0 $\pm$ 0.1       & 73.5 $\pm$ 0.7       & 76.4 $\pm$ 0.6       & 77.5                 \\
CMA     & 95.8 $\pm$ 0.4       & 65.0 $\pm$ 0.5       & 70.6 $\pm$ 2.4       & 78.1 $\pm$ 0.3       & 77.4                 \\
\bottomrule
\end{tabular}}
\end{center}

\textbf{PACS}

\begin{table}[ht!]
\begin{center}
\adjustbox{max width=\textwidth}{%
\begin{tabular}{lccccc}
\toprule
\textbf{Algorithm}   & \textbf{A}           & \textbf{C}           & \textbf{P}           & \textbf{S}           & \textbf{Avg}         \\
\midrule
ERM                  & 81.2 $\pm$ 0.9       & 73.4 $\pm$ 0.9       & 96.1 $\pm$ 0.6       & 70.3 $\pm$ 0.5       & 80.2                 \\
CORAL                & 80.6 $\pm$ 0.6       & 74.9 $\pm$ 0.2       & 95.9 $\pm$ 0.4       & 69.4 $\pm$ 0.2       & 80.2                 \\
Fishr                & 83.6 $\pm$ 0.6       & 74.9 $\pm$ 1.0       & 97.4 $\pm$ 0.3       & 70.1 $\pm$ 0.5       & 81.5                 \\
CMA     & 82.8 $\pm$ 0.7       & 76.7 $\pm$ 1.3       & 97.3 $\pm$ 0.2       & 69.5 $\pm$ 0.7       & 81.6                 \\
\bottomrule
\end{tabular}}
\end{center}
\end{table}

\textbf{TerraIncognita}
\begin{table}[ht!]
\begin{center}
\adjustbox{max width=\textwidth}{%
\begin{tabular}{lccccc}
\toprule
\textbf{Algorithm}   & \textbf{L100}        & \textbf{L38}         & \textbf{L43}         & \textbf{L46}         & \textbf{Avg}         \\
\midrule
ERM                  & 50.2 $\pm$ 0.4       & 25.0 $\pm$ 1.9       & 36.3 $\pm$ 1.6       & 34.5 $\pm$ 0.1       & 36.5                 \\
CORAL                & 43.1 $\pm$ 3.2       & 21.4 $\pm$ 2.7       & 37.5 $\pm$ 0.6       & 32.1 $\pm$ 0.5       & 33.6                 \\
Fishr                & 49.9 $\pm$ 2.1       & 23.2 $\pm$ 1.8       & 41.4 $\pm$ 1.2       & 34.7 $\pm$ 0.7       & 37.3                 \\
CMA     & 47.5 $\pm$ 3.4       & 44.7 $\pm$ 2.4       & 29.0 $\pm$ 3.2       & 32.4 $\pm$ 0.9       & 38.4                 \\
\bottomrule
\end{tabular}}
\end{center}
\end{table}

\subsubsection{Additional Baselines}
We compare CMA against additional baselines on ColoredMNIST and RotatedMNIST datasets and discuss the results using test-domain model selection. From \Cref{tab:add_base}, we observe that CMA has the second-highest average accuracy. Note that VREx surpasses CMA on ColoredMNIST but has a substantially larger variance (4.6) compared to CMA (0.9).
\begin{table}[ht!]
\begin{center}
\caption{Model selection: test-domain validation set}
\adjustbox{max width=\textwidth}{%
\begin{tabular}{lccc}
\toprule
\textbf{Algorithm}        & \textbf{ColoredMNIST}     & \textbf{RotatedMNIST}     & \textbf{Avg}              \\
\midrule
ERM                       & 54.5 $\pm$ 0.2            & 97.8 $\pm$ 0.1            & 76.2                      \\
CORAL                     & 55.7 $\pm$ 0.5            & 98.0 $\pm$ 0.0            & 76.9                      \\
Fishr                     & 62.0 $\pm$ 1.7            & 97.9 $\pm$ 0.0            & 80.0                      \\
GroupDRO                  & 59.6 $\pm$ 0.3            & 98.0 $\pm$ 0.1            & 78.8                      \\
DANN                      & 53.5 $\pm$ 0.7            & 97.4 $\pm$ 0.0            & 75.5                      \\
CDANN                     & 53.6 $\pm$ 0.4            & 97.6 $\pm$ 0.0            & 75.6                      \\
VREx                      & 66.1 $\pm$ 4.6            & 97.8 $\pm$ 0.0            & \textbf{82.0}             \\
SelfReg                   & 53.8 $\pm$ 0.8            & 98.0 $\pm$ 0.1            & 75.9                      \\
CMA                       & 62.5 $\pm$ 0.9            & 97.9 $\pm$ 0.1            & \underline{80.2}          \\
\bottomrule
\end{tabular}}
\label{tab:add_base}
\end{center}
\end{table}

\begin{table}[ht!]
\begin{center}
\caption{Model selection: Training-domain validation set}
\adjustbox{max width=\textwidth}{%
\begin{tabular}{lccc}
\toprule
\textbf{Algorithm}        & \textbf{ColoredMNIST}     & \textbf{RotatedMNIST}     & \textbf{Avg}              \\
\midrule
ERM                       & 51.7 $\pm$ 0.1            & 97.8 $\pm$ 0.1            & 75.8                      \\
CORAL                     &51.7 $\pm$ 0.1             &98.0 $\pm$ 0.0            & 74.9                      \\
Fishr                     & 52.2 $\pm$ 0.1            & 97.9 $\pm$ 0.1            & 75.1                      \\
GroupDRO                  & 51.9 $\pm$ 0.1            & 97.9 $\pm$ 0.1            & 74.9                      \\
DANN                      & 51.7 $\pm$ 0.0            & 97.6 $\pm$ 0.2            & 74.6                      \\
CDANN                     & 51.9 $\pm$ 0.2            & 97.8 $\pm$ 0.0            & 74.8                      \\
VREx                      & 51.7 $\pm$ 0.1            & 97.7 $\pm$ 0.1            & 74.7                      \\
SelfReg                   & 51.7 $\pm$ 0.0            & 98.1 $\pm$ 0.1            & 74.9                      \\
CMA                       &51.4 $\pm$ 0.0             & 97.8 $\pm$ 0.0            & 74.6                      \\
\bottomrule
\end{tabular}}
\end{center}
\end{table}

\subsection{Comparison to HGP}\label{app:compare_hgp}
We compare CMA with the HGP algorithm~\citep{hemati_understanding_2023}. Both algorithms align the gradients and Hessians, so we expect their performances to be similar. We do not compare CMA with Hutchinson's in ~\citep{hemati_understanding_2023} due to the time costs incurred by sampling-based Hessian estimation.

\subsubsection{Linear Probing}
As shown in \Cref{tab:hgp_cma_probing}, the two algorithms have comparable performance overall, except for the CelebA dataset. A potential explanation for this discrepancy is the differences in Hessian computations. Note that the original HGP does not apply penalty annealing. We added penalty annealing to both methods to eliminate differences caused by this factor, allowing us to focus on differences in the loss objectives.

\begin{table}[ht!]\label{tab:hgp_cma_probing}
    \centering
    \caption{Test accuracy (\%) with standard error. Each experiment is repeated over 5 random seeds.}
    \centering
    \resizebox{\textwidth}{!}{%
        \begin{tabular}{@{}lcccccc@{}}
            \toprule
            \textbf{Method} & \multicolumn{2}{c}{\textbf{Waterbirds (CLIP ViT-B/32)}} & \multicolumn{2}{c}{\textbf{CelebA (CLIP ViT-B/32)}} & \multicolumn{2}{c}{\textbf{MultiNLI (BERT)}}                                                                         \\
            \cmidrule(l){2-3} \cmidrule(l){4-5} \cmidrule(l){6-7}
                            & \textbf{Average}                                        & \textbf{Worst-Group}                                & \textbf{Average}                             & \textbf{Worst-Group}  & \textbf{Average}      & \textbf{Worst-Group}  \\
            \midrule
            HGP & 90.47 ± 0.06 & 86.48 ± 0.12 & 75.14 ± 0.12 & 71.68 ± 0.18 & 80.72 ± 0.62 & 69.35 ± 0.68 \\ 
            CMA & 90.11 ± 0.17 & 86.16 ± 0.29 & 77.87 ± 0.04 & 74.16 ± 0.10 & 81.30 ± 0.25 & 69.72 ± 0.66 \\
        \bottomrule
    \end{tabular}}
\end{table}

\subsubsection{Fine-tuning}
We also run fine-tuning experiments on HGP, strictly following the implementation in the code released by \citet{hemati_understanding_2023}. The hyperparameter search scheme in DomainBed leads to more uncertainty and the implementation of HGP in \citet{hemati_understanding_2023} does not employ penalty annealing. Together with the differences in the Hessian computation, all of these factors potentially lead to the differences in the performance of CMA and HGP.

\begin{table}[ht!]
\caption{Model selection: test-domain validation set}
\adjustbox{max width=\textwidth}{%
\begin{tabular}{lcccccc}
\toprule
\textbf{Algorithm}        & \textbf{ColoredMNIST}     & \textbf{RotatedMNIST}     & \textbf{VLCS}             & \textbf{PACS}             & \textbf{TerraIncognita}   & \textbf{Avg}              \\
\midrule
HGP                       & 55.8 $\pm$ 0.2            & 97.8 $\pm$ 0.1            & 76.5 $\pm$ 1.2            & 79.8 $\pm$ 0.2            & 29.6 $\pm$ 0.9            & 67.9                      \\
CMA				          & 62.5 $\pm$ 0.9            & 97.9 $\pm$ 0.1            & 77.4 $\pm$ 0.8            & 81.6 $\pm$ 0.3            & 38.4 $\pm$ 1.2            & 71.5                      \\
\bottomrule
\end{tabular}}
\end{table}

\begin{table}[ht!]
\caption{Model selection: training-domain validation set}
\adjustbox{max width=\textwidth}{%
\begin{tabular}{lcccccc}
\toprule
\textbf{Algorithm}        & \textbf{ColoredMNIST}     & \textbf{RotatedMNIST}     & \textbf{VLCS}             & \textbf{PACS}             & \textbf{TerraIncognita}   & \textbf{Avg}              \\
\midrule
HGP                       & 51.8 $\pm$ 0.0            & 97.9 $\pm$ 0.1            & 75.8 $\pm$ 1.0            & 77.5 $\pm$ 1.0            & 28.6 $\pm$ 0.8            & 66.3                      \\
CMA				          & 51.4 $\pm$ 0.0            & 97.8 $\pm$ 0.0            & 77.9 $\pm$ 0.1            & 80.5 $\pm$ 0.2            & 34.0 $\pm$ 0.7            & 68.3                      \\
\bottomrule
\end{tabular}}
\end{table}

\end{document}